\documentclass[sigconf]{acmart}

\AtBeginDocument{%
  }

\setcopyright{acmlicensed}
\copyrightyear{2018}
\acmYear{2018}
\acmDOI{XXXXXXX.XXXXXXX}
\acmConference[Conference acronym 'XX]{Make sure to enter the correct
  conference title from your rights confirmation email}{June 03--05,
  2018}{Woodstock, NY}
\acmISBN{978-1-4503-XXXX-X/2018/06}

\usepackage{url}

\usepackage{times}  
\usepackage{helvet}  
\usepackage{courier}  

\usepackage{graphicx} 
\usepackage{bbm}

\usepackage{amsmath,amssymb,amsfonts}
\usepackage{graphicx}
\usepackage{textcomp}
\usepackage{xcolor}
\usepackage{multirow}
\usepackage{subfigure}

\usepackage[ruled,linesnumbered,lined,commentsnumbered,algo2e]{algorithm2e}
\usepackage{algorithm}
\usepackage{algorithmic}
 
\usepackage{enumitem}

\usepackage[table]{xcolor}

\usepackage{url}            
\usepackage{booktabs}       
\usepackage{threeparttable}
\usepackage{amsfonts}       
\usepackage{nicefrac}       
\usepackage{microtype}      

\usepackage{mathtools}
\usepackage{amsthm}
\usepackage{mathrsfs}
\usepackage{epstopdf}

\usepackage{newfloat}
\usepackage{listings}

\theoremstyle{plain}
\newtheorem{theorem}{Theorem}[section]
\newtheorem{proposition}[theorem]{Proposition}

\theoremstyle{definition}
\newtheorem{definition}[theorem]{Definition}

\theoremstyle{remark}

\begin{document}

\title{Learning to Unlearn: Machine Unlearning via Learning the Unlearning Behaviors
}

\author{Hang Zhang}
\affiliation{%
  \institution{State Key Laboratory for Novel Software Technology \\ Nanjing University}
  \city{Nanjing}
  \country{China}}
\email{zhanghang@lamda.nju.edu.cn}

\author{Kaifeng Zhang}
\affiliation{%
  \institution{State Key Laboratory for Novel Software Technology \\ Nanjing University}
  \city{Nanjing}
  \country{China}}
\email{zhangkf2022@lamda.nju.edu.cn}

\author{Yixiao Ma}
\affiliation{%
  \institution{State Key Laboratory for Novel Software Technology \\ Nanjing University}
  \city{Nanjing}
  \country{China}}
\email{mayx@lamda.nju.edu.cn}

\author{Wei-Jie Xu}
\affiliation{%
  \institution{State Key Laboratory for Novel Software Technology \\ Nanjing University}
  \city{Nanjing}
  \country{China}}
\email{xuwj@lamda.nju.edu.cn}

\author{Ye Zhu}
\affiliation{%
  \institution{Centre for Cyber Resilience and Trust \\ Deakin University}
  \city{Burwood, VIC}
  \country{Australia}}
\email{ye.zhu@ieee.org}

\author{Kai Ming Ting}
\correspondingauthor
\affiliation{%
  \institution{State Key Laboratory for Novel Software Technology \\ Nanjing University}
  \city{Nanjing}
  \country{China}}
\email{tingkm@nju.edu.cn}

\renewcommand{\shortauthors}{Zhang et al.}


\begin{abstract}
 Various machine unlearning techniques have been developed in response to privacy legislation requirements, enabling individuals to exercise their legal right to have their data $D_f$ removed from a machine learning model. This process is typically accomplished via the use of an unlearning function denoted as $U$. Existing methods focus on designing an intricate $U$ to unlearn $D_f \subset D$ from a previous model $A(D)$, so that the unlearned model performs as closely as possible to the retrained model $A(D \setminus D_f)$. However, these methods often suffer from high computational costs when dealing with massive training data, as the complex structures of $U$ become a bottleneck even for models with fewer parameters.
 Inspired by Learning to Optimize, we introduce the first learning-based model-agnostic approach, Learning-to-UnLearn (L2UL). Our core insight is to shift from manually designing $U$ to learning the unlearning behaviors from a distribution perspective, thereby acquiring a simple and efficient $U$ via learning. Our experimental results demonstrate that the accuracy achieved by L2UL is comparable to that of retraining while exhibiting impressive efficiency, particularly in data-intensive scenarios. Furthermore, we validate the performance and scalability of our method on larger models ResNet.
\end{abstract}

\ccsdesc[500]{Security and privacy}
\ccsdesc[500]{Computing methodologies~Machine learning}
\ccsdesc[300]{Computing methodologies~Artificial intelligence}

\keywords{
Machine unlearning, Data privacy, Privacy-preserving learning.
}


\received{20 February 2007}
\received[revised]{12 March 2009}
\received[accepted]{5 June 2009}

\maketitle

\section{Introduction}

With the progress of society, there has been a significant increase in the amount of available data. This growth has enriched our lives by allowing us to access a wide range of data shared by numerous individuals online. However, it also presents risks to the privacy of users. In response to this, legislative measures like CCPA (California Consumer Privacy Act), PIPEDA (Personal Information Protection and Electronic Documents Act), and GDPR (General Data Protection Regulation) have been introduced to establish rules that protect users' privacy by the right to delete their data. The task of erasing data from a machine learning model is named machine unlearning \cite{cao2015towards,SISA,surveyofmu}. 

A naive approach is to retrain the model on the remaining data, but this is expensive as models are often trained on very large datasets. Proposing efficient and effective machine unlearning methods has attracted the attention of researchers in recent years. 

Current approaches focus on designing a complex unlearning function $U$ so that the unlearned model performs close to the retrained model.
For example, some methods design $U$ as a retraining on all or part of the data, which is very time-consuming and requires retraining every time on an unlearning request. Some methods necessitate the calculation of the Hessian matrix for the designed $U$, which further increases the computational cost. Consequently, even for relatively simple models, the unlearning efficiency remains bottlenecked by the complexity of the designed $U$, especially when navigating massive datasets.

The goal of an unlearning function $U$ is to produce a model that has unlearned the user data and performs close to the model retrained on the remaining data. \textit{\textbf{Why don't we learn to unlearn directly?} }

Inspired by Learning to Optimize (L2O) \cite{chen2022learning,tang2024learn}, traditional optimization algorithms like Adam are manually designed by human experts through theoretical derivation. In contrast, the L2O approach aims to let machines automatically learn how to optimize. Specifically, the L2O model takes the current state of the optimization iteration as input (such as the current solution and its gradient) and learns to predict the update amount for the next iteration.
\textit{Instead of designing a complex and highly time-consuming $U$, we use the retrained model as ground truth to learn a $U$ whose structure is simple but enough to unlearn $D_f$ effectively}. In this way, we can achieve the goal of unlearning much more efficiently.
Although many methods claim learning to unlearn \cite{cha2024learning, ma2022learn}, they usually refer to learning a new decision boundary to achieve forgetting rather than learning the behavior of unlearning.

In this paper, we introduce a model-agnostic framework called Learning-to-UnLearn (L2UL), which employs a machine-learning approach to learn a $U$.

The main contributions of this work are:
\begin{enumerate}
    \item Pioneering a new framework of machine unlearning via learning, named Learning-to-UnLearn (L2UL). 
    \item Providing the generalization bound of a model unlearned by L2UL in the context of Logistic Regression and Multi-Layer Perceptron (MLP).
    \item Conducting an empirical validation of the effectiveness and efficiency of our proposed method.
\end{enumerate}

The advantages of the proposed L2UL over current methods are two-fold: 
(i) Unlike existing methods, in L2UL, the acquisition of $U$ is achieved through learning rather than designing.
(ii) Due to the simple structure of $U$, the execution time of unlearning via the learned $U$ is very short.

\section{Problem Formulation}
Notations used in this paper are summarized in Table \ref{tab: notations}.

\begin{table}[!ht]
    \centering
    \caption{Key notations used in the paper.}
    \label{tab: notations}
    \scalebox{1}{
    \begin{tabular}{ll}
    \toprule 
    \textbf{Symbols} & \textbf{Definition}   \\ \midrule
    $\mathcal{X}$    &  Sample space         \\
    $\mathcal{Y}$    &  Label space          \\
    $D$              &  Training set  \\
    $D_f$            &  Forget set  \\ \midrule
    $A(\cdot)$       &  Learning algorithm   \\
    $U(\cdot)$       &  Unlearning algorithm \\      
    $\theta$         &  Parameters of the model trained on $D$    \\    
    $\theta'$   &  Parameters of the model trained on $D\setminus D_f$ \\ \midrule
    $\mathcal{P}$    & Probability distribution    \\
    $x$              &  A vector in $\mathbb{R}^d$ \\
    $\kappa$    &   Kernel Function\\
    $\kappa_I$  &  Isolation Kernel \\
    $t,\psi$              &  Parameter of $\kappa_I$ \\
    \bottomrule 
    \end{tabular}}
\end{table}

Let $\mathcal{X}$ denote the sample space, and $\mathcal{Y}$ denote the label space. A data point is an element in $\mathcal{Z}=\mathcal{X}\times \mathcal{Y}$. A hypothesis function $h: \mathcal{X}\rightarrow\mathcal{Y}$ is learned on training set $D\subset \mathcal{Z}$ by algorithm $A$ to assign $y\in\mathcal{Y}$ to $x\in \mathcal{X}$.

\begin{definition}
    A machine learning algorithm is a map $A$ from a subset of $\mathcal{Z}$ to a hypothesis function $h\in\mathcal{H}$, where $\mathcal{H}$ is the space of all hypothesis functions of the algorithm.
\end{definition}

After $A$ learned a model from a dataset $D$, $A(D)$ can be applied to a test set. When some users request to have their data $D_f\subset D$ deleted, model $A(D)$ is obliged to have unlearned their data. The task of unlearning data from a learned model is machine unlearning.

\begin{figure*}[!htb]
    \centering
    \includegraphics[width=0.9\linewidth]{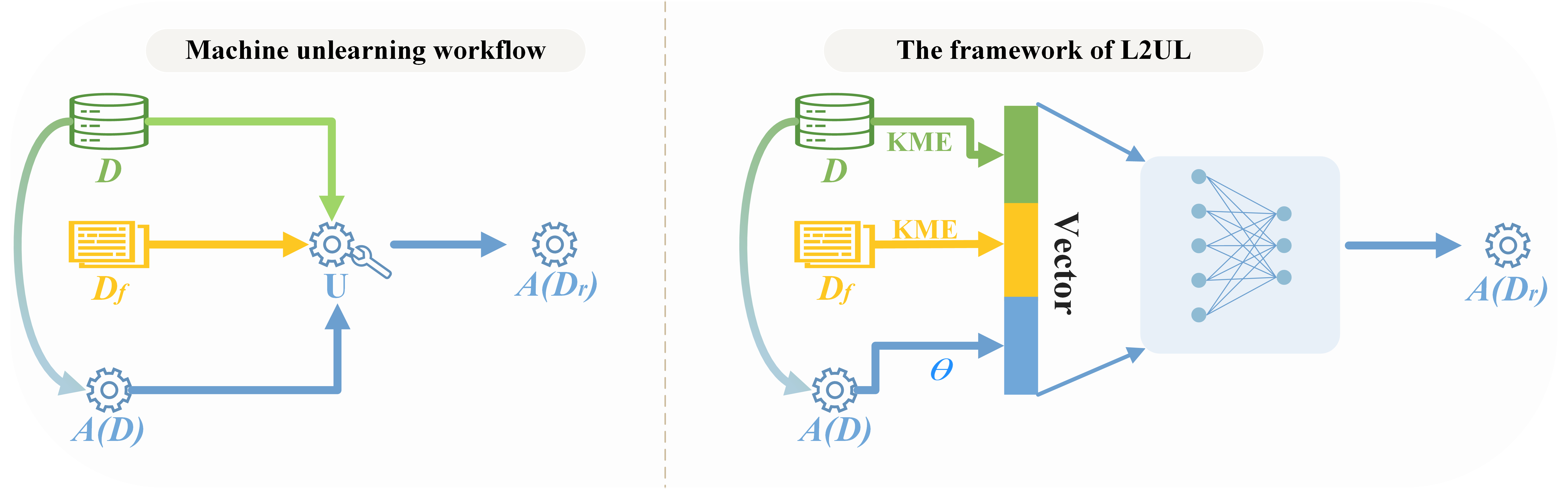}
    \caption{The Machine unlearning workflow and the framework of our proposed method L2UL. The model $A(D)$ is learned on dataset $D$ (which includes $D_f$) by $A$. An unlearning algorithm $U$ unlearns $D_f$ from $A(D)$ to obtain a revised model, which is expected to be equal to the retrained model $A(D\setminus D_f)$.}
    \label{fig:algo}
\end{figure*}

Figure \ref{fig:algo} shows the workflow of machine unlearning. 
To fulfill a request to unlearn $D_f$ from $A(D)$, an unlearning algorithm $U(D, D_f, A(D))$ produces a model which is expected to be the same or close to the model $A(D\setminus D_f)$ retrained on the dataset $D\setminus D_f$.

\begin{definition}
    A machine unlearning algorithm is a map $U:2^{\mathcal{Z}}\times 2^{\mathcal{Z}}\times \mathcal{H}\rightarrow \mathcal{H}$, where $2^{\mathcal{Z}}$ is the power set of $Z$, which is set containing all of $\mathcal{Z}$'s subset. $U$ takes $D, D_f, A(D)$ as input, and outputs the unlearned model, which is expected to be as close to the retrained model $A(D\setminus D_f)$ as possible. 
    \label{def: U}
\end{definition}

\begin{definition}
    Exact Unlearning: given a randomized learning algorithm $A$, a dataset $D$, and a data subset $D_f\subset D$ to be forgotten, an unlearning algorithm $U$ is an exact unlearning process if and only if $\mathcal{P}(U(D, D_f, A(D)))=\mathcal{P}(A(D\setminus D_f))$, where $\mathcal{P}(A(D))$ defines the distribution of all models trained on a dataset $D$ by a learning algorithm $A$. 
\end{definition}

The naive unlearning algorithm simply retrains a new model on $D\setminus D_f$ to achieve exact unlearning. However, it is expensive to retrain when $|D\setminus D_f|$ is large.

\section{Related Work}
\label{sect: related work}
Many machine unlearning algorithms that remove some user data from a previously trained model without retraining have been proposed recently. According to the workflow shown in Figure \ref{fig:algo}, we divide the related work into three categories: $D$-oriented, $D_f$-oriented, and $A(D)$-oriented unlearning.

\textbf{$D$-oriented unlearning} operates on subsets of $D$ such that only retaining on some subsets is required upon a request to unlearn $D_f$.
For example, SISA \cite{SISA} divides the training data $D$ into shards and divides each shard into slices. Each shard is used to train a constituent model, incrementally incorporating slices and preserving its parameters until the training set is extended with a new slice. When $D_f$ is required to be forgotten, retraining is limited to the specific constituent model whose shards encompass $D_f$. Many methods share a similar intuition, such as \cite{cao2015towards, dare2021, kdd2023tree, unlearnkmeans2019, gupta2021adaptive, wu2023deltaboost}.

\textbf{$D_f$-oriented unlearning} considers the impact of forgotten data on the model, through influence functions \cite{wu2022puma, guo2020certified, warnecke2021machinelabels} and Markov Chain Monte Carlo-based sampling \cite{2022MCU, fu2021knowledge}. When $D_f$ needs to be forgotten, the impact of $D_f$ on the model can be erased accordingly.

\textbf{$A(D)$-oriented unlearning} considers how the parameters of a learned model should be updated after $D_f$ is forgotten from the training set \cite{sekhari2021remember, neel2021descent, ma2022learntoforget, tarun2023deep, wu2020deltagrad}. \cite{tarun2023fastyet} updates the parameters by adding noise to the parameters to forget and then tuning the model on the remaining data $D\setminus D_f$. Few-shot learning \cite{chundawat2023zero} is used in the case where
$D$ is unavailable, and only model $A(D)$ is known. 

Although existing works can produce an unlearned model close to the retrained model, none of them use a learning technique to obtain $U$. 
Our proposed L2UL, an $A(D)$-oriented unlearning algorithm, tries to learn the unlearning function $U$, instead of designing one like current works.

\section{Proposed Method}

\subsection{The Framework of L2UL}

Recall Definition \ref{def: U} and Figure \ref{fig:algo}, $U$ has three inputs, $D$ (represented  as a $|D|\times d$ matrix), $D_f$ (represented  as a $|D_f|\times d$ matrix) and the parameter vector $\theta$ of model $A(D)$.
A machine unlearning algorithm $U$ must satisfy the following properties:

\begin{proposition}
The output of $U$ should remain unchanged after any two rows in the data matrix of $D$ or $D_f$ are exchanged. 
\label{prop: exchange}
\end{proposition}
This is easy to understand because exchanging any two rows of $D$ and $D_f$ does not change the dataset.

An effective method that satisfies the Proposition \ref{prop: exchange} is to represent $D$ and $D_f$ as distribution $\mathcal{P}_D$ and $\mathcal{P}_{D_f}$. No matter how the order of points of the dataset changes, its distribution remains the same.
Therefore, we redefine Definition \ref{def: U} as:
\begin{definition}
        A machine unlearning algorithm is a function $U$: $\mathscr{P}\times \mathscr{P}\times \mathcal{H} \rightarrow\mathcal{H}$, where $\mathscr{P}$ is the space containing all the probability distributions. 
        \label{def: U_P}
\end{definition}

Different from Definition \ref{def: U}, Definition \ref{def: U_P} defines machine unlearning from the perspective of distribution.
Based on Definition \ref{def: U_P} and the unlearning workflow shown in Figure \ref{fig:algo}, L2UL is naturally proposed by using three components of the unlearning workflow: $D, D_f$, and $A(D)$. Figure \ref{fig:algo} shows the framework of L2UL, which uses a neural network $N_U$\footnote{Other models can also be used. A two-layer (256, 64) fully connected neural network is used in our experiments.} to produce the unlearned model $A_{\theta^\prime}$. 

In order to vectorize distribution $\mathcal{P}_D$ and $\mathcal{P}_{D_f}$ as input to $N_U$, we employ Kernel Mean Embedding (KME) \cite{KME2007} to transform distribution $\mathcal{P}_D$ and $\mathcal{P}_{D_f}$ to vector
$\mu_D$ and $\mu_f$ in Reproducing Kernel Hilbert Space (RKHS).

Kernel mean embedding is a nonparametric method that uses an element of an RKHS as the representation of a probability distribution. The mean feature map $\mu_\mathbb{P}=\mathbb{E}[\kappa(\cdot, D)]$ of points $X$ based on kernel $\kappa$ embeds distribution $\mathbb{P}$ into RKHS, and preserves all of the statistical features of $\mathbb{P}$. 

In our framework, when using some common kernels such as the Gaussian kernel, the Laplacian kernel, etc., $\mu_\mathbb{P}$ will map $\mathbb{P}$ into an infinite-dimensional space, which cannot be used as input to the neural network. In other words, we need a kernel whose RKHS dimension is finite.
One approach is to use low-rank representations of the kernel matrix. The most popular examples are Nystr{\"o}m method \cite{kumar2012sampling,williams2000using,musco2017recursive}  and the Random Fourier Features \cite{rahimi2007random,yang2012nystrom}. However, these are all approximate methods and will reduce effectiveness.

Another approach is to use a kernel with an exact and finite-dimensional feature map. Isolation Kernel \cite{ting2018isolation} is such a kernel that directly obtains the feature map without calculating the kernel matrix, which has been widely used and performs well \cite{zhang2025kernel,ting2022new,cao2025anomaly}.
Let $D\subset \mathbb{R}^d$ be a dataset sampled from an unknown distribution $\mathbb{P}$. Isolation Kernel uses a partition mechanism such as Voronoi Diagram \cite{qin2019nearest, pmlr-v202-zhang23bb} or hypersphere \cite{ting2020-IDK} to partition the data into $|\mathcal{D}|$ cells, where $|\mathcal{D}|=\psi$, and $\mathcal{D}\subset D$ is sampled from $D$ as the seeds of partition mechanism. Let $\mathbb{H}_\psi(D)$ denotes the set of all partitions $H$, each cell $\mathscr{I}[z]$ of partition $ H$ isolates $z\in \mathcal{D}$ from the rest of the points in $\mathcal{D}$, the definition of Isolation Kernel is shown as follows:

\begin{definition}
    Isolation Kernel: $\forall x,y\in \mathbb{R}^d$, Isolation Kernel of $x$ and $y$ is defined to be the probability that $x$ and $y$ fall into the same cell $\mathscr{I}[z]$ of partition $H$ over all the partitions $H\in \mathbb{H}_\psi(D)$:\\
    $\kappa_I(x,y|D)$ = $\mathbb{E}_{\mathbb{H}_{\psi}(D)}[\mathbf{1}(x,y\in \mathscr{I}[z]|\mathscr{I}[z]\in H)]$
\end{definition}

Given the $t$ partitions $H_1,...,H_t$, the feature map $\Phi(x)$ of $\kappa_I$ is a $\psi\times t$-dimensional binary column vector.

\begin{definition}
Isolation Kernel's feature map, denoted as $\Phi(x)$, is a vector that represents the cell in which $x$ falls into over partition of $t$ times. For $x\in \mathbb{R}^d$, the dimension of $\Phi(x)$ is $\psi\times t$, where $\psi$ is the number of cells in one partition $H$. Each element of $\Phi(x)$ is either 0 or 1, indicating the cell in which $x$ falls.
\end{definition}

Since the dimension of $\Phi(x)$ is finite, the KME of $\kappa_I$ can be obtained by simply computing the mean of all feature maps.

Then $\mu_D$, $\mu_f$, and the parameters $\theta$ of model $A$ are concatenated as the input $\mathbf{x}$ of the neural network, which contains two hidden layers. 
The KME of distribution $\mathbb{P}$ on the given data $D$ is: 
    \begin{equation}
        \mu_\mathbb{P} = \mathbb{E}[\kappa_I(\cdot, D)]=\frac{\sum_{x\in D} \Phi(x)}{|D|},
        \label{eqa: kme}
    \end{equation}
where $\Phi(X)$ is IK's finite-dimensional feature map determined by two hyperparameters $t,\psi$. This is also known as the isolation distribution kernel (IDK)\cite{ting2020-IDK}. According to Equation \ref{eqa: kme}, we have: $\mu_D = \frac{\sum_{x\in D} \Phi(x)}{|D|}$ and $\mu_f = \frac{\sum_{x\in D_f} \Phi(x)}{|D_f|}.$ 
$\mu_D$ and $\mu_f$ are later concatenated with the parameters of the original model as input $\mathbf{x}$.

The output $\theta'$ is the parameter of the unlearned model $U(D, D_f, A(D))$. And Mean Squared Error (MSE) loss is used as the objective function for training $U$:
\begin{equation}
    \mathcal{L}=\frac{1}{|\theta^\prime|}\sum_{i=1}^{|\theta^\prime|}(\theta_i^\prime-\theta^r_i)^2,
    \label{eqa: mse}
\end{equation}
where $\theta^r$ is the parameters of retrained model $A(D\setminus D_f)$.

Intuitively, we are focusing on models with relatively few parameters but very large training datasets. Because the model has few parameters, $U$ can be learned efficiently, and the large amount of training data can be represented via kernel mean embedding. Therefore, L2UL is effective and highly efficient.

\subsection{How to Learn $U$}

In order to train $U$, first we need a training set $\mathscr{D} = \mathscr{X}\times\mathscr{Y}$. 
We produce $\mathscr{D}$ by the following three steps. First, we generate $\mathfrak{D}$ by randomly sampling $s$ points from $D$ ($s\ll|D|$), and train a model $A(\mathfrak{D})$ on $\mathfrak{D}$. Second, we randomly select some samples as a forget dataset $\mathfrak{D}_f\subset \mathfrak{D}$. $\mu_{\mathfrak{D}}$, $\mu_{\mathfrak{D}_f}$ of $\mathfrak{D}$ and $\mathfrak{D}_f$ are obtained using Equation \ref{eqa: kme}. The parameter $\theta$ of model $A(\mathfrak{D})$ is concatenated with $\mu_{\mathfrak{D}}$ and $\mu_{\mathfrak{D}_f}$ as $\mathbf{x}$. Finally, we obtain $\theta^r$, the parameter of retrained model $A(\mathfrak{D}\setminus\mathfrak{D}_f)$. Now we have one instance $(\mathbf{x},\theta^r)\in\mathscr{X}\times\mathscr{Y}$.

The above steps are repeated $m$ times to obtain $\mathscr{D}$ containing $m$ samples. $N_U$ is later trained on $\mathscr{D}$ by minimizing the loss function $\mathcal{L}$ shown in Equation \ref{eqa: mse}. Since the parameters in KME do not require learning, learning $N_U$ is equivalent to learning $U$. When there is a $D_f$ that needs to be unlearned from model $A(D)$, we can obtain the unlearned model $A_{\theta^\prime}$ via the learned $N_U$.

An important point to highlight is that once we have obtained $N_U$, if we need to unlearn any user data $D_f$, we can obtain the unlearned model $U(D, D_f, A(D))$ via the learned $N_u$ directly without having to repeat the preprocessing and training phase.

The pseudo code of preprocessing for obtaining $\mathscr{D}$ and $\Phi(\cdot)$ for L2UL, and unlearning $D_f$ are shown in Algorithm \ref{alg: preprocessing} and Algorithm \ref{alg: unlearn}, respectively.

\begin{algorithm2e}
    \SetKwData{Left}{left}\SetKwData{This}{this}\SetKwData{Up}{up}
    \SetKwFunction{Union}{Union}\SetKwFunction{FindCompress}{FindCompress}
    \SetKwInOut{Input}{Input}\SetKwInOut{Output}{Output}
    \Input{$D$: Dataset, $m$: Number of Samples in $\mathscr{D}$, $s$: size of $\mathfrak{D}$, $\psi$, $t$: Parameters of $\kappa_I$}
    \Output{$\mathscr{D}$, $\Phi(\cdot)$}
    $\mathscr{D}$ = \{\}\;
    Get the map function $\Phi(\cdot)$ using $\kappa_I$. \;
    $\forall x\in D$, get the feature map $\Phi(x)$\;
    \For{$i\leftarrow 1$ \KwTo $m$}{
    $\mathfrak{D} \leftarrow$ randomly sample $s$ points from $D$\;
    $\mathfrak{D}_f \leftarrow$ randomly sample from $\mathfrak{D}$ \;
    $A(\mathfrak{D}), A(\mathfrak{D}\setminus \mathfrak{D}_f) \leftarrow$ train A on $\mathfrak{D}$, $\mathfrak{D}\setminus \mathfrak{D}_f$\;
    $\theta,\theta^r \leftarrow$ Extract parameters of $A(\mathfrak{D}), A(\mathfrak{D}\setminus \mathfrak{D}_f)$\;
    $\mu_\mathfrak{D}, \mu_f \leftarrow$ obtain KME using Equation 1 \;
    $\mathbf{x} \leftarrow$ concatenation of $\mu_\mathfrak{D}, \mu_f$ and $\theta$ \;
    $\mathscr{D}\leftarrow$ $\mathscr{D} \cup \{(\mathbf{x},\theta^r)\}$\;
    }
    Return $\mathscr{D}$, $\Phi(\cdot)$.
\caption{Preprocessing}\label{alg: preprocessing}
\end{algorithm2e}

\begin{algorithm2e}
    \SetKwData{Left}{left}\SetKwData{This}{this}\SetKwData{Up}{up}
    \SetKwFunction{Union}{Union}\SetKwFunction{FindCompress}{FindCompress}
    \SetKwInOut{Input}{Input}\SetKwInOut{Output}{Output}
    \Input{$N_U, \mu_D$, $\Phi(\cdot)$, $A(D)$}
    \Output{$A_{\theta'}$}
    $\forall x\in D_f$, get the feature map $\Phi(x)$\;
    $\mu_f \leftarrow$ obtain KME of $D_f$ through Equation 1\;
    $\mathbf{x} \leftarrow$ concatenation of $\mu_D, \mu_f$ and $A(D)$'s parameter $\theta$ \;
    $\theta' \leftarrow N_U(\mathbf{x})$\;
    Return unlearned model $A_{\theta'}$.
\caption{Unlearn $D_f$}\label{alg: unlearn}
\end{algorithm2e}

\subsection{Generalization Bound Analysis}
Here, we provided the generalization bound of the classifier (Logistic Regression, Multilayer Perceptron) unlearned by L2UL.

Denote the expected loss of Logistic Regression model $A(x|\theta)=\frac{e^{\theta x+b}}{1+e^{\theta x+b}}$ as: $L(A)=E_{x,y}[y\log(A(x))+(1-y)\log(1-A(x)]$.

L2UL learns the unlearning function $U$. The expected loss of the model unlearned by $U$ is upper-bounded, as shown in the following Theorem \ref{theorem: logistic Unlearn error}. 

\begin{theorem}
    \begin{equation*}
         L({U(A(D), D, D_f)})\leq L({A(D\setminus D_f)}) + 2R\sqrt{\epsilon},
    \end{equation*} 
    where $R$ is the radius of the dataset $X$ ($\forall x\in X, ||x||\leq R$), and $\epsilon$ is the generalization error bound (MSE) of the unlearning model $U$ in L2UL.
    \label{theorem: logistic Unlearn error}
\end{theorem}

Similarly, the expected loss of Multilayer Perceptron (MLP) unlearned by $U$  is also upper-bounded, as shown in Theorem \ref{theorem: MLP Unlearn error}.

\begin{theorem}
    \begin{equation*}
        L({U(A(D), D, D_f)})\leq L({A(D\setminus D_f)}) + C\sqrt{\epsilon},
    \end{equation*} 
    where $L$ is cross entropy loss, $C$ is a finite scalar determined by the MLP structure and $\epsilon$ is the generalization bound (MSE) of the unlearning model $U$ in L2UL.
    \label{theorem: MLP Unlearn error}
\end{theorem}

Both Theorem \ref{theorem: logistic Unlearn error} and \ref{theorem: MLP Unlearn error} indicate that the expected loss of the classifier unlearned by L2UL is close to that of retrained classifier.

\section{Experiments}

\textbf{System:} The experiments are executed on a Linux machine with 1T GB RAM and an AMD 128-core CPU, with each core running at 2 GHz.

\textbf{Data:} We use seven public datasets in our experiments\footnote{The datasets are available at \url{https://archive.ics.uci.edu/}, and \url{https://www.csie.ntu.edu.tw/~cjlin/libsvmtools/datasets/}}. The specifications of the datasets are summarised in Table \ref{tab: datasets}. 

\begin{table}[!ht]
\centering
\caption{Dataset Summary. n=no.instances, d=no.attributes.}
\label{tab: datasets}
\setlength{\tabcolsep}{0.7mm}{
\begin{tabular}{l|ccccccc}
\toprule
    & Magic      & Adult      & Sepsus   & Skin   & Covetype  & SUSY       & HIGGS  \\   
\midrule
n          & 19020      & 32562      & 40328    & 245057 & 581012  & 5000000   & 11000000  \\
d          & 10          & 14          & 89        & 3       & 54       & 18          & 28        \\
\bottomrule 
\end{tabular}}
\end{table}

\textbf{Comparison algorithms:} We compare L2UL\footnote{The codes are available at \url{https://anonymous.4open.science/r/L2UL}} with Retrain, one $D$-oriented unlearning algorithm: SISA and two $A(D)$-oriented unlearning algorithms: DeltaGrad, FYEMU.

\begin{enumerate}
    \item \textbf{Retrain:} As the most naive approach, it just retrains the model on the remaining dataset $D\setminus D_f$. 
    \item \textbf{SISA:} SISA \cite{SISA} is a well-known framework that partitions the data to reduce the retraining time by just retraining a single model on the shard that needs to be forgotten.
    \item \textbf{DeltaGrad:} DeltaGrad \cite{wu2020deltagrad} unlearns the data by differentiating the optimization path with the Quasi-Newton method based on information cached during the training phase.
    \item \textbf{FYEMU:} FYEMU \cite{tarun2023fastyet} is a unlearning algorithm for Neural Networks, which first unlearns data by adding noise to the model parameters, and then obtains a new model through fine-tuning.
\end{enumerate}

\textbf{Evaluation:} We evaluate the efficiency and effectiveness of unlearning algorithms using unlearning time and the accuracy of the unlearned classifier on test data ($20\%$), which is averaged over 10 runs. In each run, we unlearn one randomly selected instance\footnote{$|D_f|=1$ for LR, $|D_f|=1$, $|D_f|=100$, and $|D_f|=1000$ for MLP.}.

\subsection{Unlearn linear classifier: Logistic Regression}

To begin, we assess the performance of our L2UL approach using a linear classifier known as Logistic Regression (LR). We have demonstrated that the generalization error of the unlearned classifier, denoted as $U(D, D_f, A(D))$, is bounded. Additionally, here we find that L2UL is both efficient and effective when applied to LR.

\begin{figure*}[h!] 
\centerline{\includegraphics[width=0.5\textwidth]{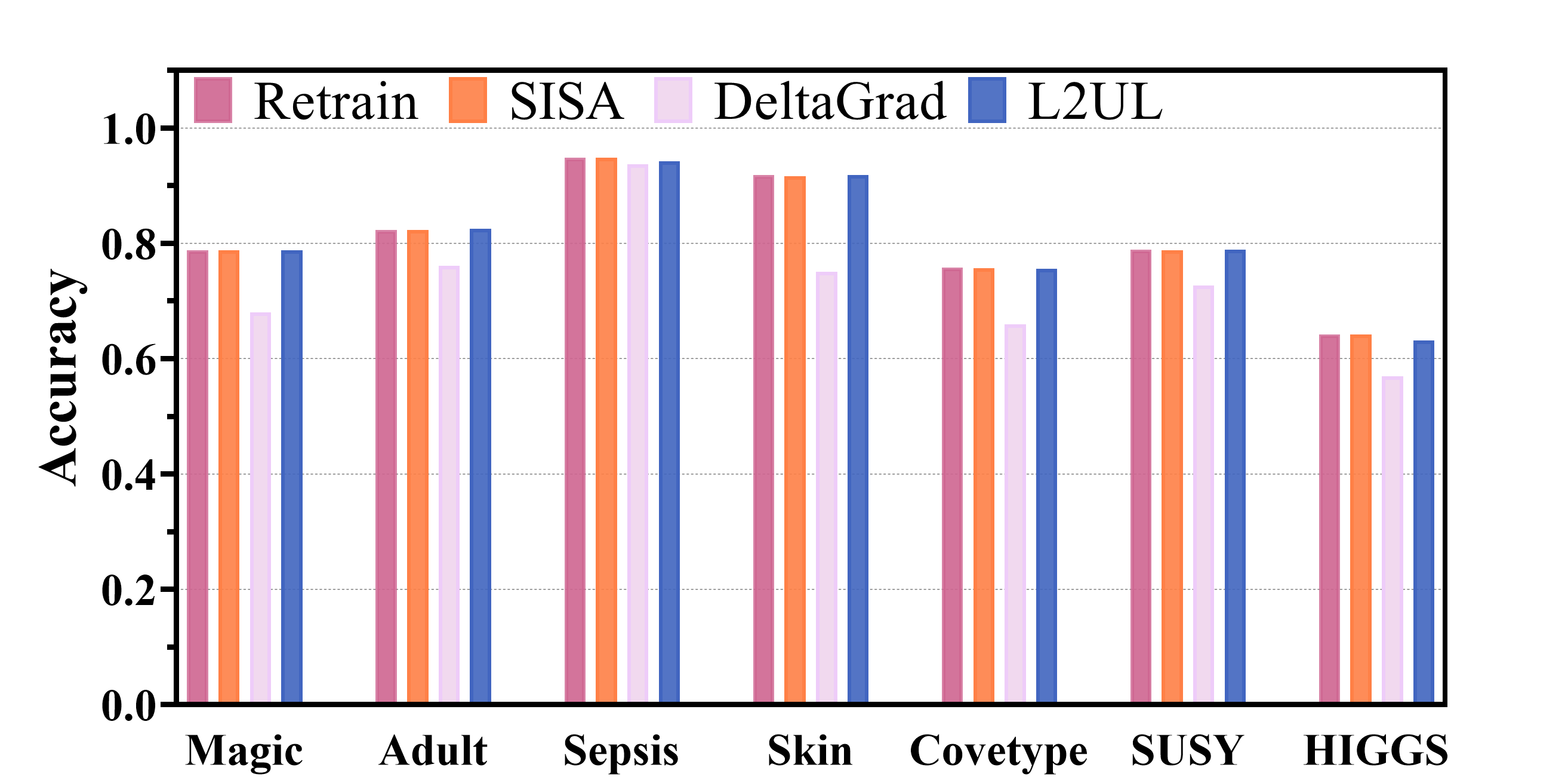}
\includegraphics[width=0.5\textwidth]{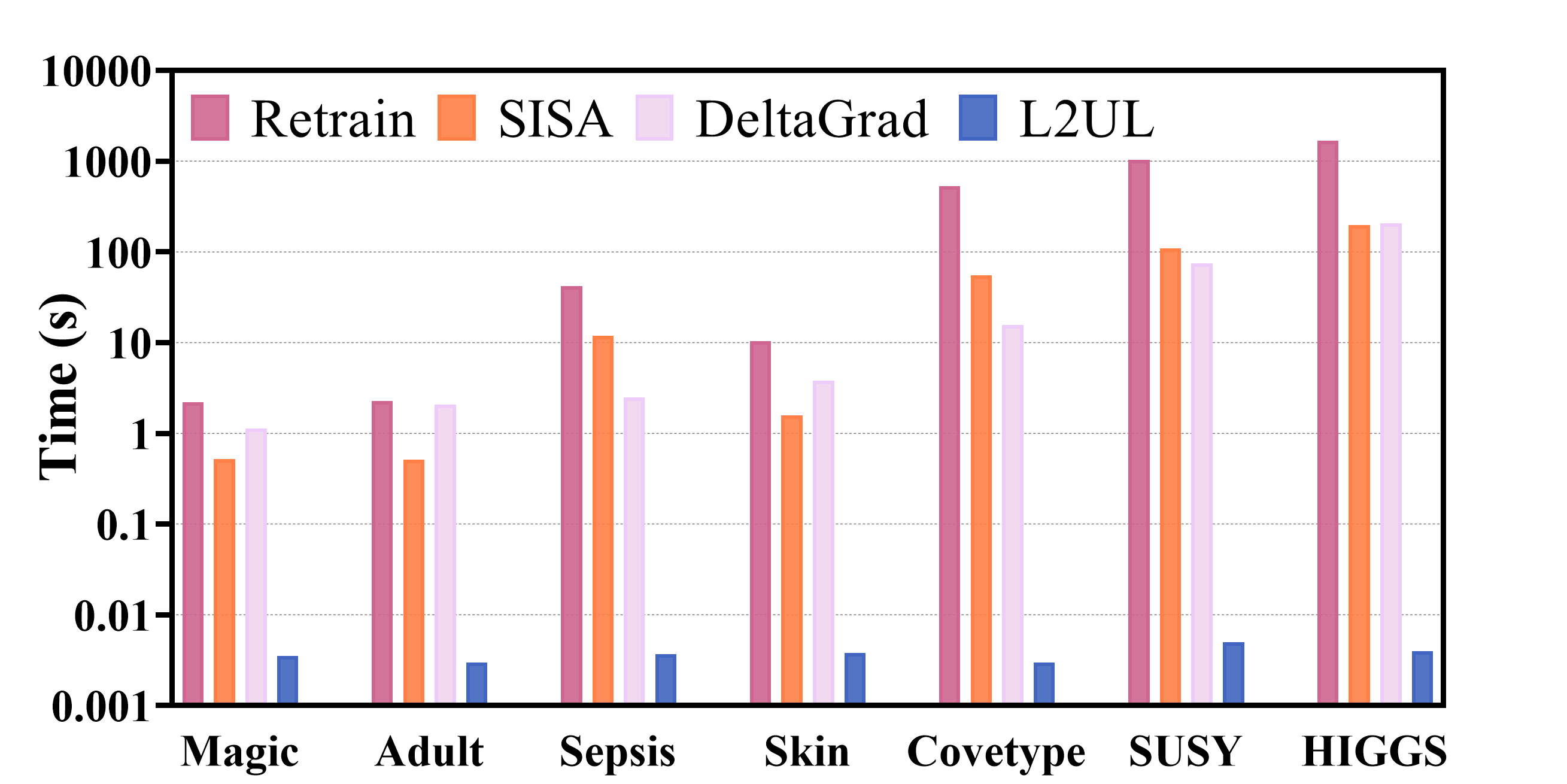}}
\caption{Results of unlearning Logistic Regression.}
\label{fig: ACC_LR_1} 
\end{figure*}

The results of unlearning one instance in terms of accuracy and unlearning time are shown in Figure \ref{fig: ACC_LR_1}. L2UL achieves comparable accuracy to Retrain and SISA and outperforms DeltaGrad over all seven datasets. 
Compared with Retain, SISA, and DeltaGrad, L2UL speeds up 630, 149, and 282 times on the Magic dataset, and 422500, 49715, and 51500 times on the HIGGS dataset, respectively.

\subsection{Unlearn non-linear classifier: Multilayer Perceptron}

Nonlinear classifiers can learn complex nonlinear boundaries, which require a more complex process to learn. Besides, a nonlinear classifier usually has more parameters, which makes its parameter space very large. Implementing unlearning in this large space is a more challenging problem.

In this subsection, we test the performance of L2UL on  Multilayer Perceptron (MLP) \cite{taud2018multilayer}. We use cross-entropy as the loss of the MLP, which has 1 hidden layer with 10 neurons.

\begin{figure*}[!ht] 
\centerline{
\includegraphics[width=0.5\textwidth]{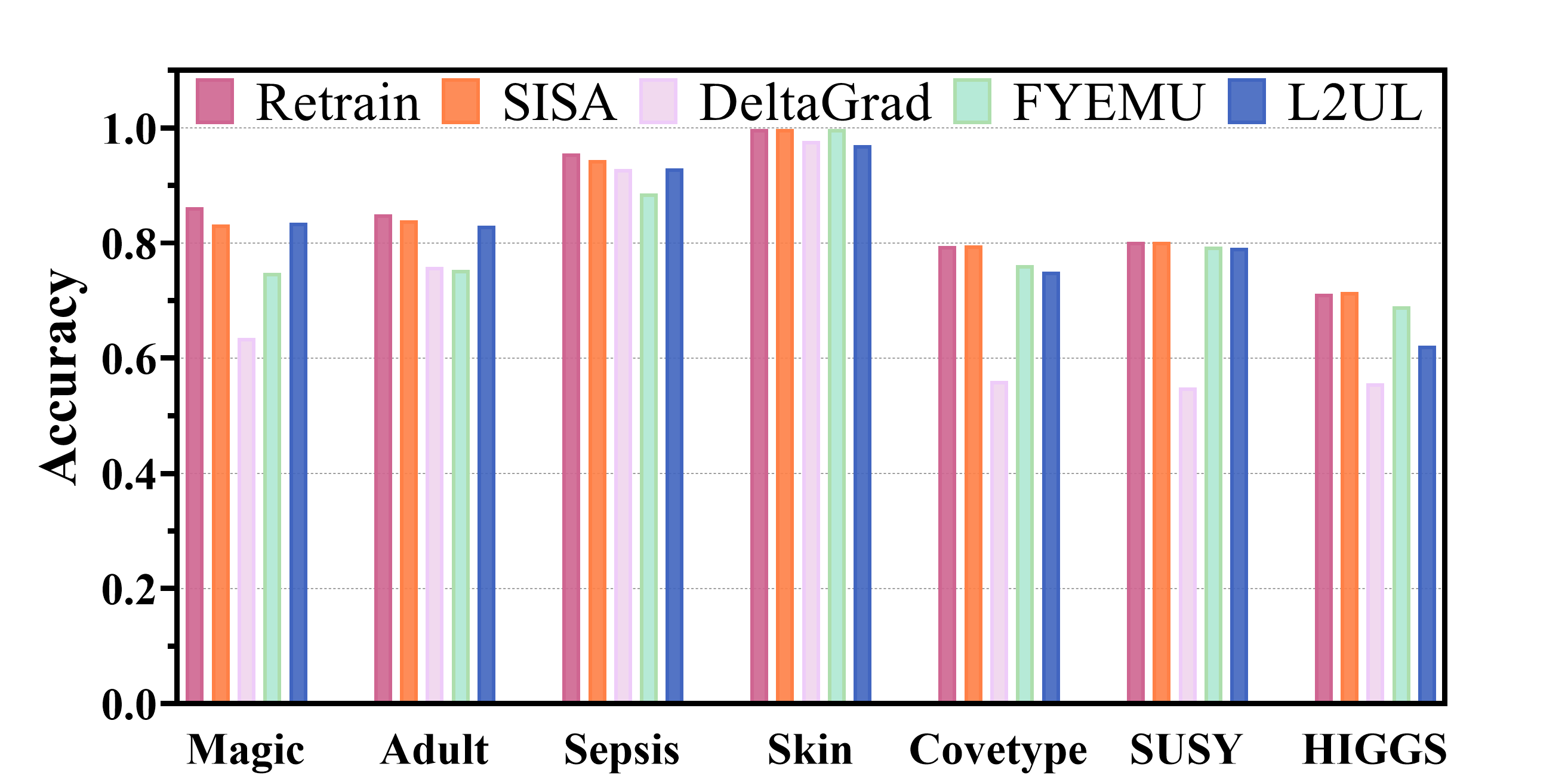}
\includegraphics[width=0.5\textwidth]{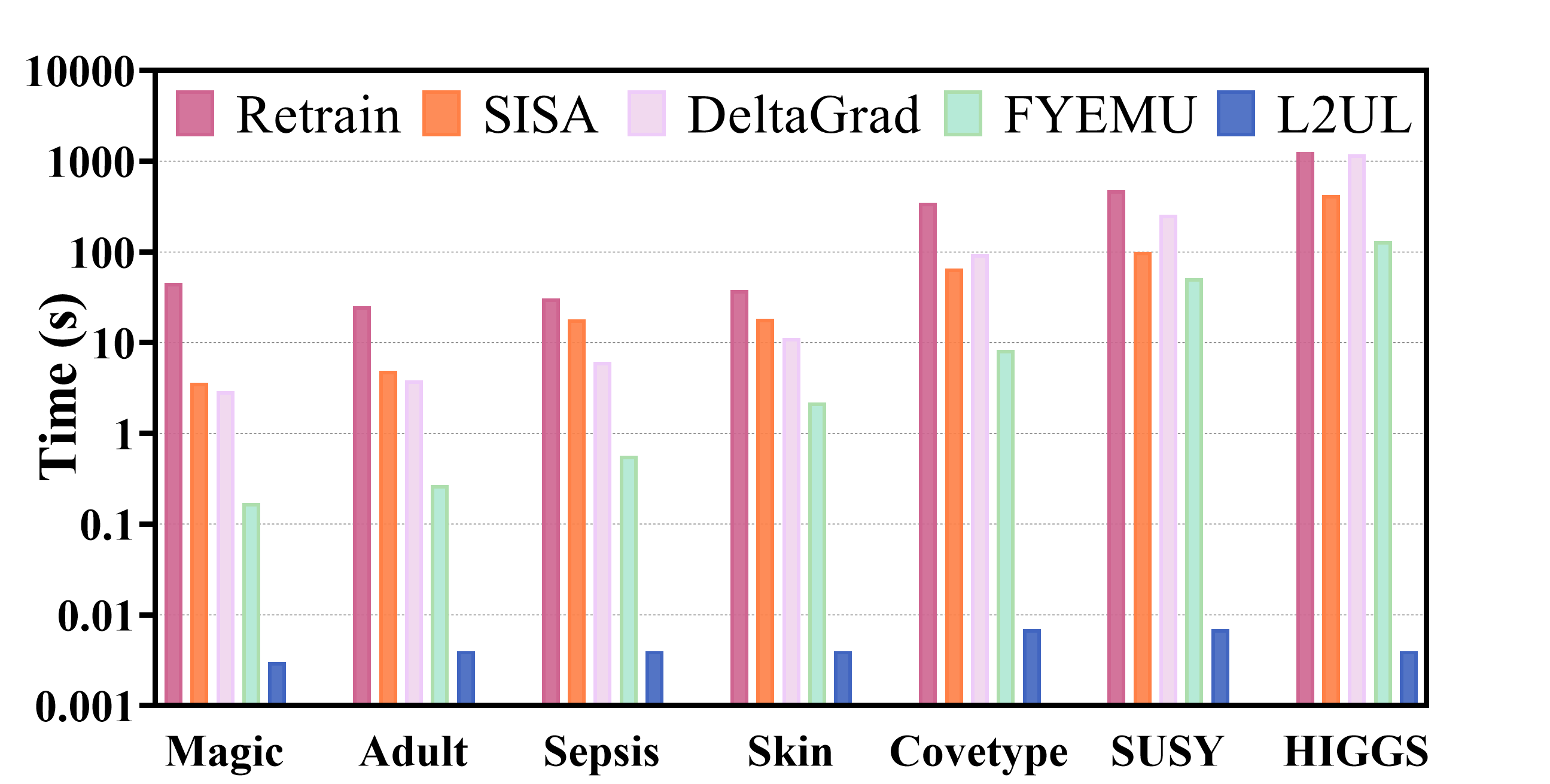} }
\caption{Results of unlearning MLP ($|D_f|=1$).} 
\label{fig: ACC_MLP_1} 
\end{figure*}

The results of unlearning one instance are shown in Figure \ref{fig: ACC_MLP_1}. L2UL achieves comparable accuracy to Retrain and SISA except HIGGS (we will discuss this in Section \ref{sect: PSA}). L2UL achieves close accuracy to Deltagrad on Sepsis and Skin and outperforms DeltaGrad on five other datasets. L2UL outperforms FYEMU on the two smallest datasets (Magic and Adult) and achieves close accuracy to FYEMU on five other datasets.

Compared with Retain, SISA, DeltaGrad, and FYEMU, L2UL speeds up 15087, 1203, 730, and 43 times, respectively, on the Magic dataset, and 398000, 105250, 294750, and 32750 times on the HIGGS dataset.
The efficiency of L2UL comes from the simple structure of $U$, which ensures that a single run of $U$ is very quick. Moreover, once $U$ is learned, 
the parameters of $U$ do not need to change with high probability when a new $D_f$ comes. We will discuss this in Section \ref{sec: disscussion}. 
L2UL just directly uses the learned $U$ to unlearn the new $D_f$, while the other algorithms require the costly computation again to unlearn $D_f$. This difference will be more significant when multiple unlearning requests come in sequence.

The results of unlearning 100 and 1000 instances in terms of accuracy and unlearning time are shown in Figure \ref{fig: ACC_MLP_100} and Figure \ref{fig: ACC_MLP_1000}.

\begin{figure*}[!ht] 
\centerline{ 
\includegraphics[width=0.5\textwidth]{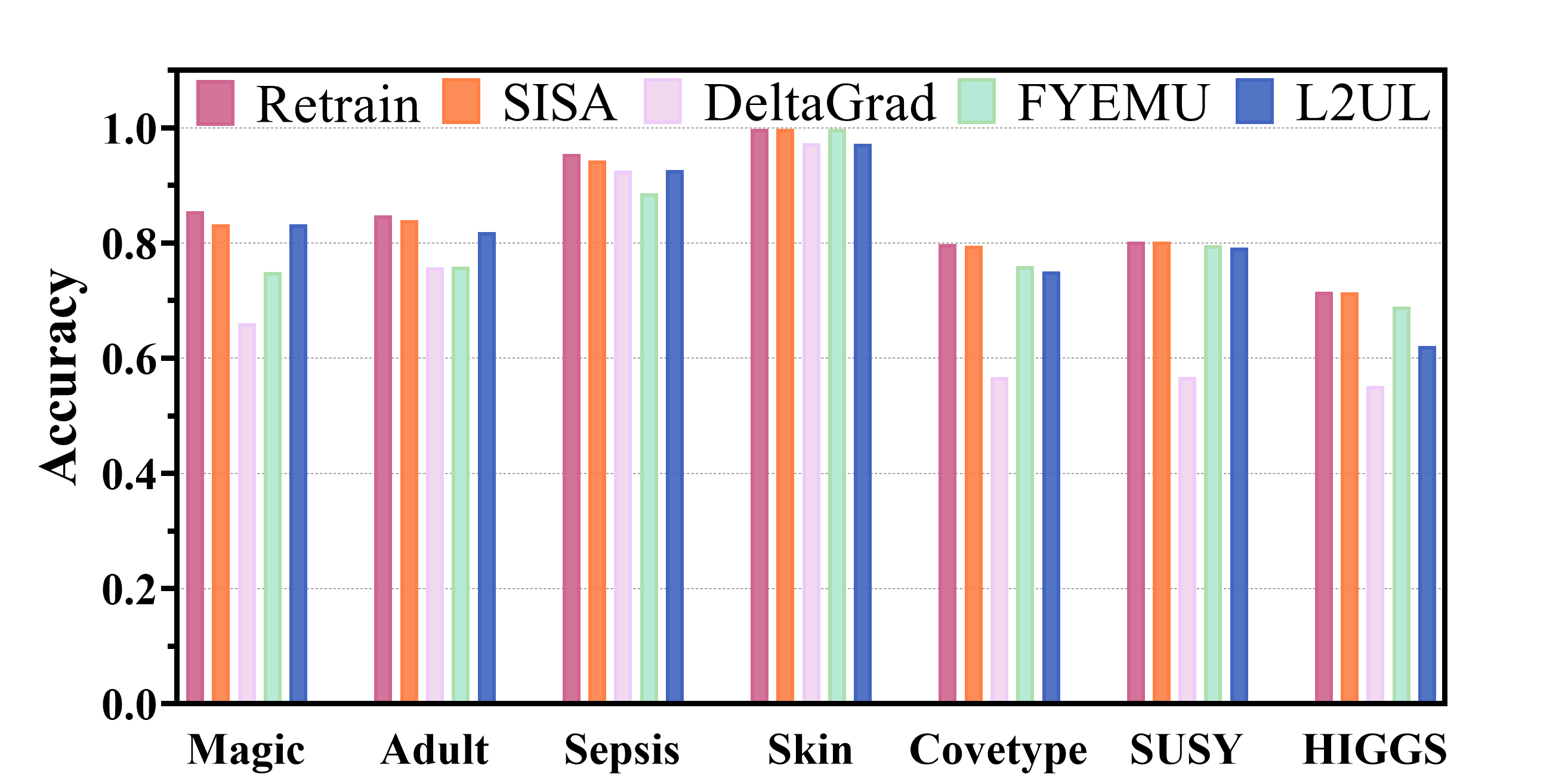} 
\includegraphics[width=0.5\textwidth]{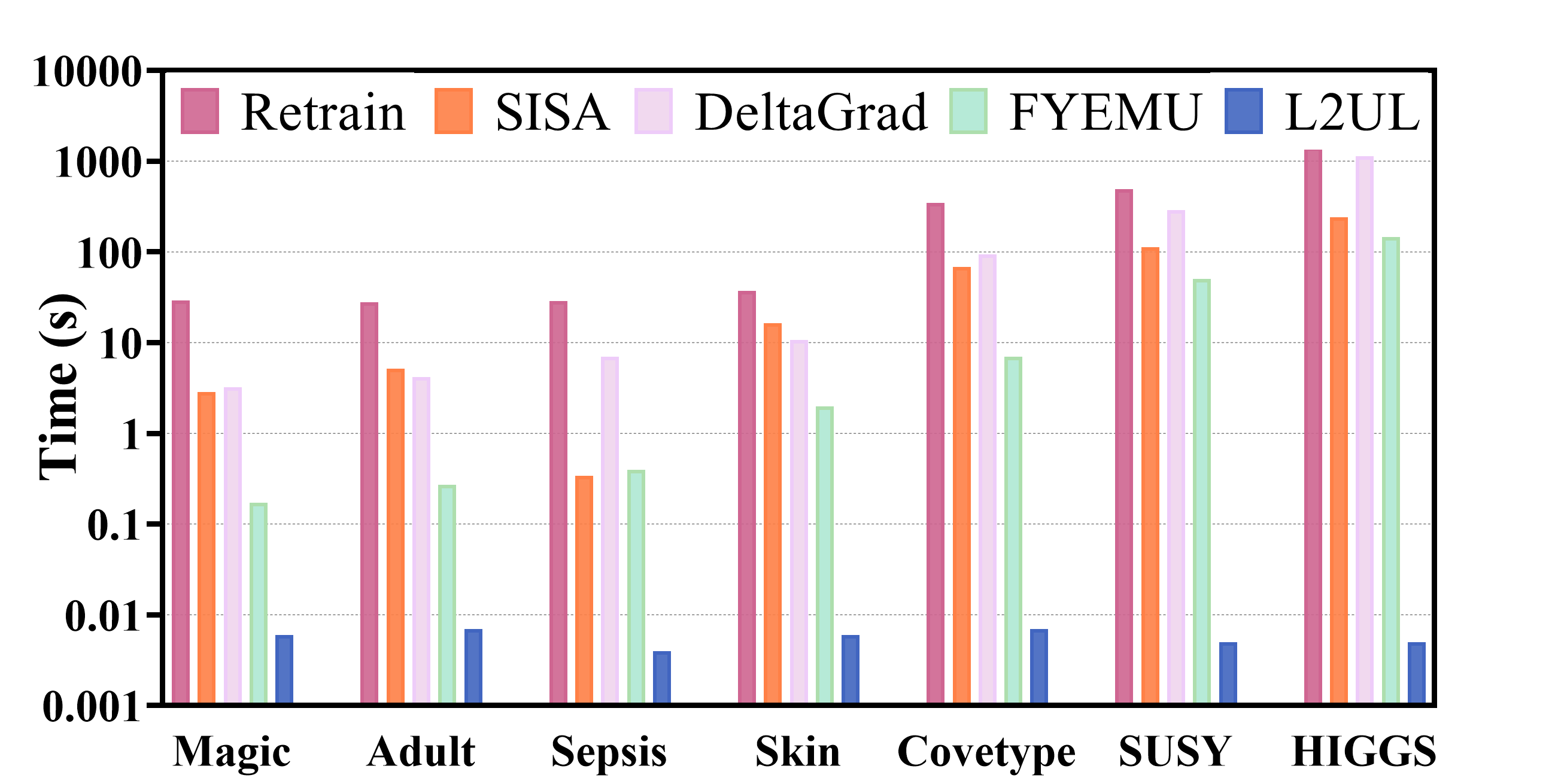}}
\caption{Results of unlearning MLP ($|D_f|=100$).} 
\label{fig: ACC_MLP_100} 
\end{figure*}

\begin{figure*}[!ht] 
\centerline{
\includegraphics[width=0.5\textwidth]{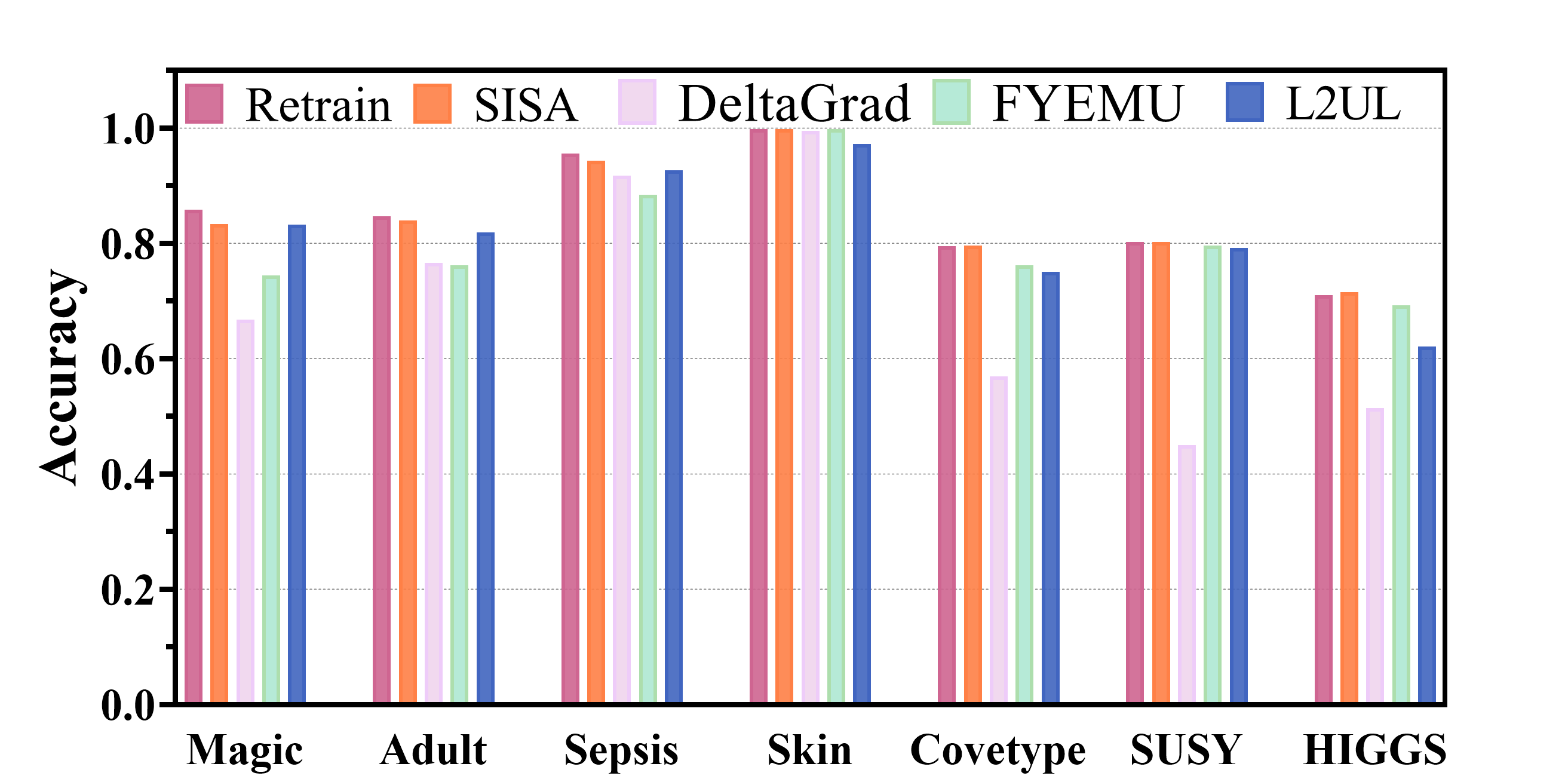} 
\includegraphics[width=0.5\textwidth]{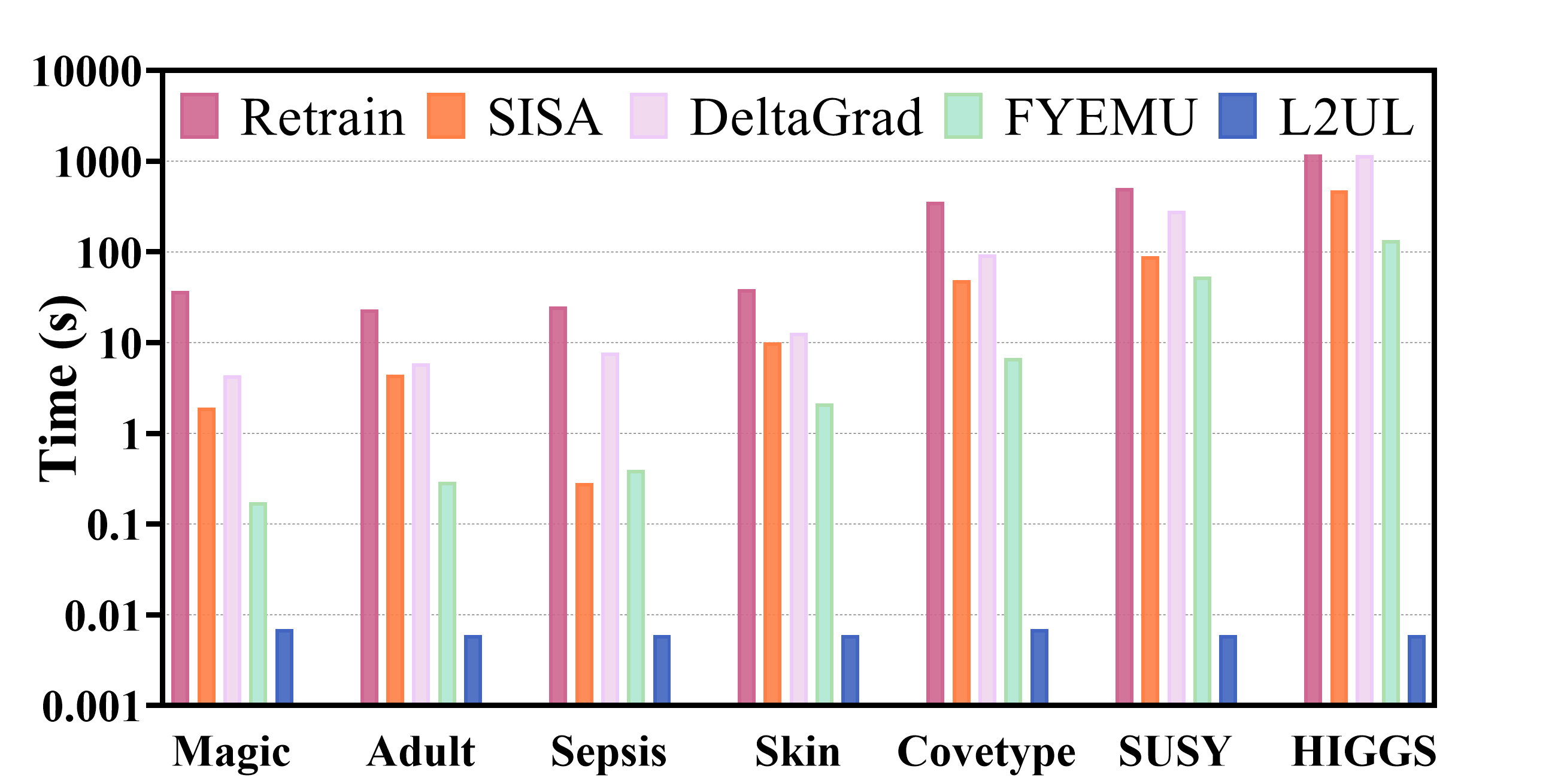}}
\caption{Results of unlearning MLP ($|D_f|=1000$).} 
\label{fig: ACC_MLP_1000} 
\end{figure*}

\subsection{Unlearning Efficacy}

In addition to effectiveness and efficiency, the machine unlearning algorithm must also be evaluated to determine whether it has forgotten users' information. We first show the two very commonly used evaluation metrics. Membership Inference Attack \cite{chundawat2023can, liu2024model} (as used in Table \ref{tab:resnetcifar10}) and accuracy on unlearned data. The results in terms of Membership Inference Attack
and accuracy on $D_f$ ($|D_f|$=1000 for MLP) in Table \ref{tabel: metric2} show that L2UL has forgotten the information of $D_f$ from $A(D)$\footnote{L2UL classifies all samples into one class on the Sepsis dataset, so MIA cannot be performed}. The $U$ has learned the unlearning behaviors. 

\begin{table*}[!ht]
\begin{threeparttable}
\centering
\caption{Results in terms of Membership Inference Attack and accuracy on unlearn data.}
\label{tabel: metric2}
\setlength{\tabcolsep}{6mm}{
\begin{tabular}{lccccccc}
\toprule
         & \multicolumn{3}{c}{MIA} & \multicolumn{4}{c}{Accuracy on $D_f$} \\
\cmidrule(lr){2-4} \cmidrule(lr){5-8}
Datasets & Retrain      & FYEMU      & L2UL      & Retrain & SISA & FYEMU & L2UL \\ \midrule
Magic    & 0.79         & 0.65       & 0.75      & 0.85    & 0.79 & 0.71  & 0.78 \\
Adult    & 0.80          & 0.68       & 0.74      & 0.81    & 0.78 & 0.75  & 0.76 \\
Skin     & 1.00            & 1.00          & 1.00         & 1.00       & 1.00    & 1.00     & 1.00    \\
Covetype & 0.73         & 0.68       & 0.67      & 0.82    & 0.76 & 0.76  & 0.71 \\
SUSY     & 0.72         & 0.69       & 0.69      & 0.73    & 0.71 & 0.70   & 0.71 \\
HIGGS    & 0.65         & 0.68$^\ast$       & 0.50       & 0.74    & 0.77$^\ast$ & 0.69 & 0.63 \\
\bottomrule
\end{tabular}}
\begin{tablenotes}
    \footnotesize 
    \item $\ast$ indicates that the results show that complete forgetting of the data was not achieved.
\end{tablenotes}
\end{threeparttable}
\end{table*}

However, it is not the case that the lower these two scores are, the better the model performance is, because this may give rise to other issues, such as information exposure \cite{golatkar2020eternal} \footnote{In this paper, if the scores are lower than both the original model ($A(D)$, ORI) and Retrain, it means the unlearning algorithm has forgotten the information.}. 
Therefore, we use an example to demonstrate the ability of the proposed algorithm L2UL for calibrating contaminated models, which also shows that L2UL has forgotten $D_f$.

As previously stated, training data can sometimes be contaminated, which negatively impacts the model's performance. Machine Unlearning is an effective method for cleaning the model when dirty data is detected. A simple example is shown in Figure \ref{fig: clean_data}. The artificial dataset has two classes, which can be classified by a linear model. But when the training set is contaminated (Figure \ref{fig: dirty_data} right shows an example with 50 contaminated points), the accuracy of the model will drop.

\begin{figure*}[!ht]
	\centering
	\subfigure[Clean data]{\includegraphics[width=0.24\linewidth]{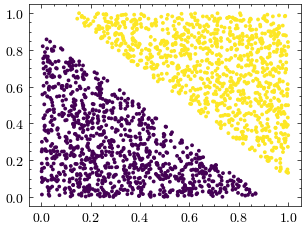}\label{fig: clean_data}}
 \subfigure[Contaminated data]{\includegraphics[width=0.24\linewidth]{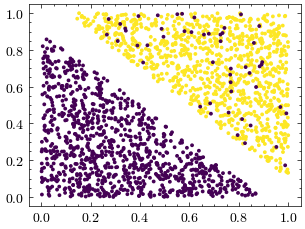}\label{fig: dirty_data}}
	\subfigure[Accuracy of LR]{\includegraphics[width=.24\linewidth]{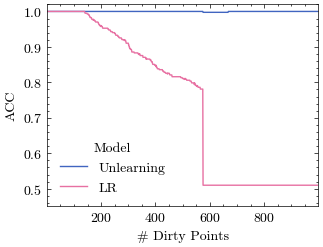}\label{fig: LR_dirty}}
	\subfigure[Accuracy of MLP]{\includegraphics[width=.24\linewidth]{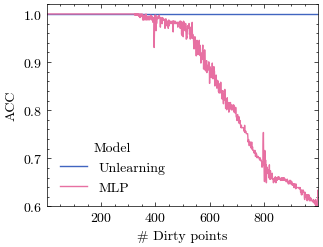}\label{fig: MLP_dirty}}
\caption{Artificial dataset with clean data and contaminated data, and the results of LR and MLP.}
\end{figure*}

Figure \ref{fig: LR_dirty} shows the results of LR in terms of Accuracy. When data is not contaminated, the accuracy is 1, as expected. As the number of dirty points grows, the accuracy declines to approximately 0.5.
When we use L2UL to clean the contaminated model by treating the dirty points as $D_f$, the accuracy of the unlearning model remains 1.00 as the number of dirty points grows.
Similar results are shown in Figure \ref{fig: MLP_dirty} when we replace LR with MLP. The difference is that the accuracy and F1 of MLP decrease more slowly than LR. The accuracy and F1 of the unlearned model are always 1.00.
This shows that L2UL has indeed achieved the unlearning of dirty data.

\subsection{Parameter Sensitivity Analysis}
\label{sect: PSA}

There are four hyperparameters in the Preprocessing and Unlearning phase of L2UL: $\psi,t$ for Isolation Kernel and $s$, $m$ for generating  $\mathscr{X}\times\mathscr{Y}$.
In our experiment, the default setting of hyperparameters is: $\psi=4$, $t=100$, $m=1000$, $s=1000$ for LR, and $s=3000$ for MLP.

We report the accuracy of unlearning 1, 100, and 1,000 instances from LR and MLP on the largest datasets SUSY and HIGGS, with $\psi\in$ [2, 4, 8, 16, 32, 64, 128, 256]\cite{ting2020-IDK}. \cite{ting2020-IDK} shows that the Isolation Distritbuion Kernel is not sensitive to $t$ ,  hence, the sensitivity analysis of $t$ is omitted here.

\begin{figure*}[!ht]
    \centering
    \includegraphics[width=0.32\linewidth]{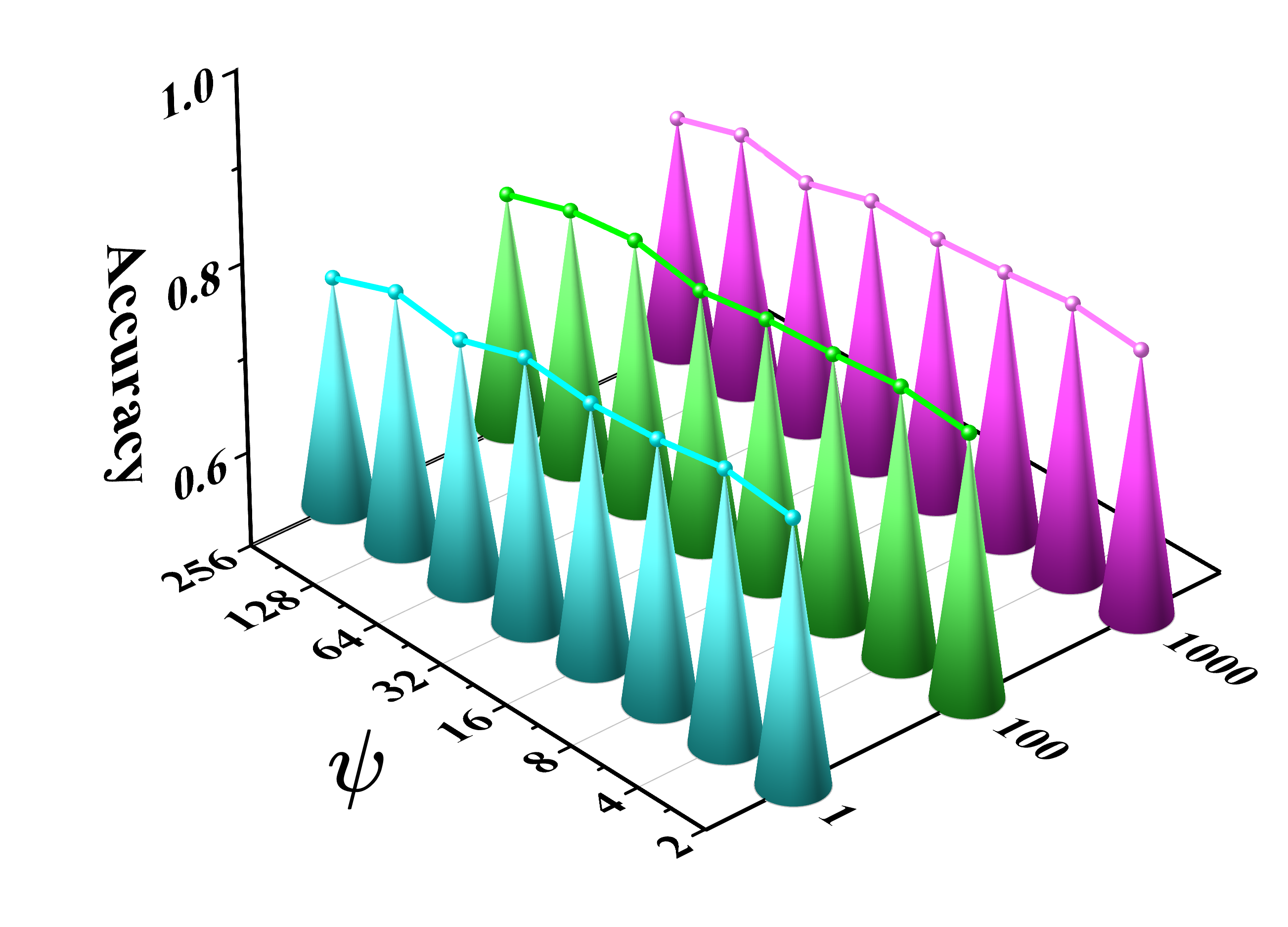}
    \includegraphics[width=0.32\linewidth]{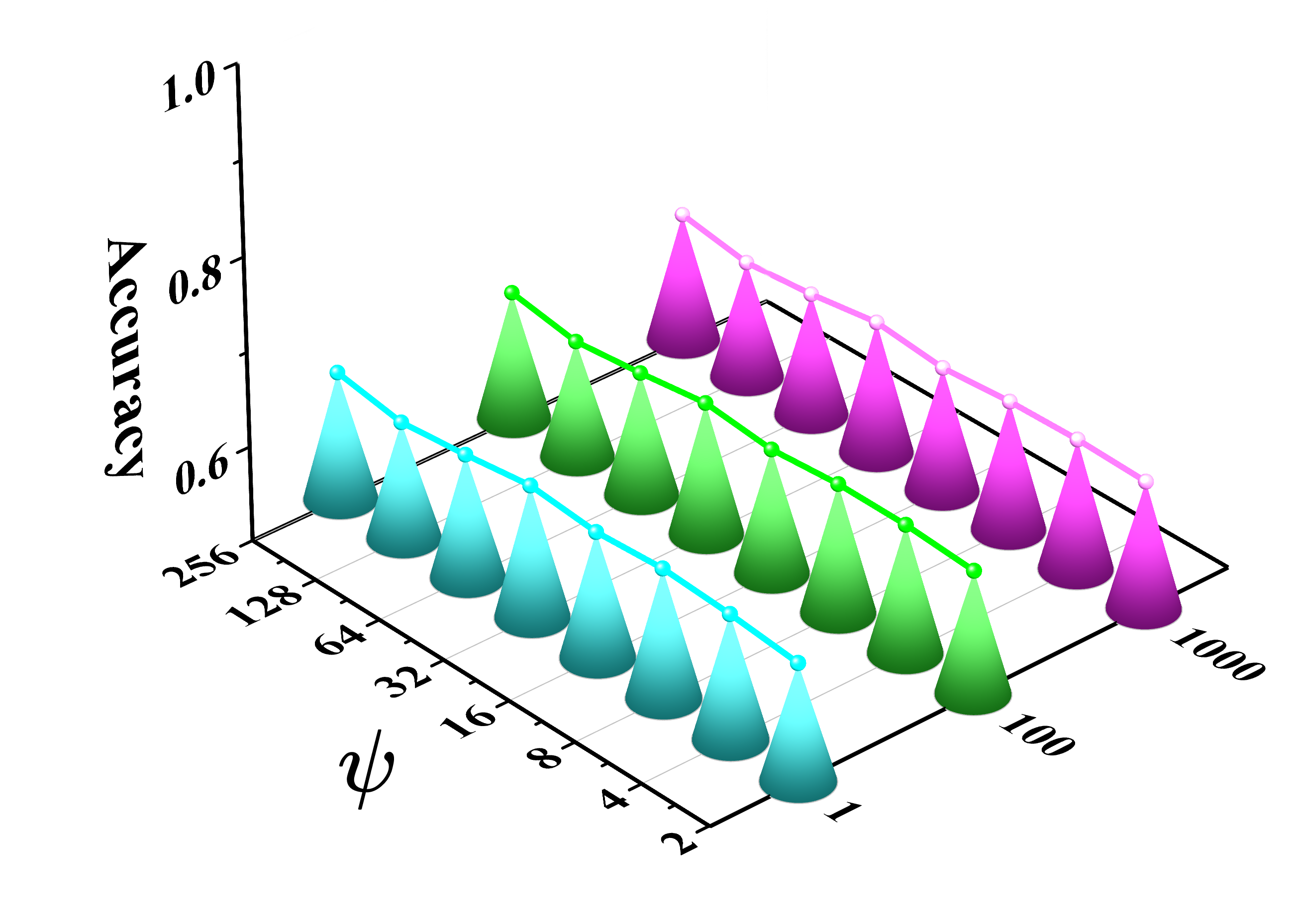}
    \includegraphics[width=0.32\linewidth]{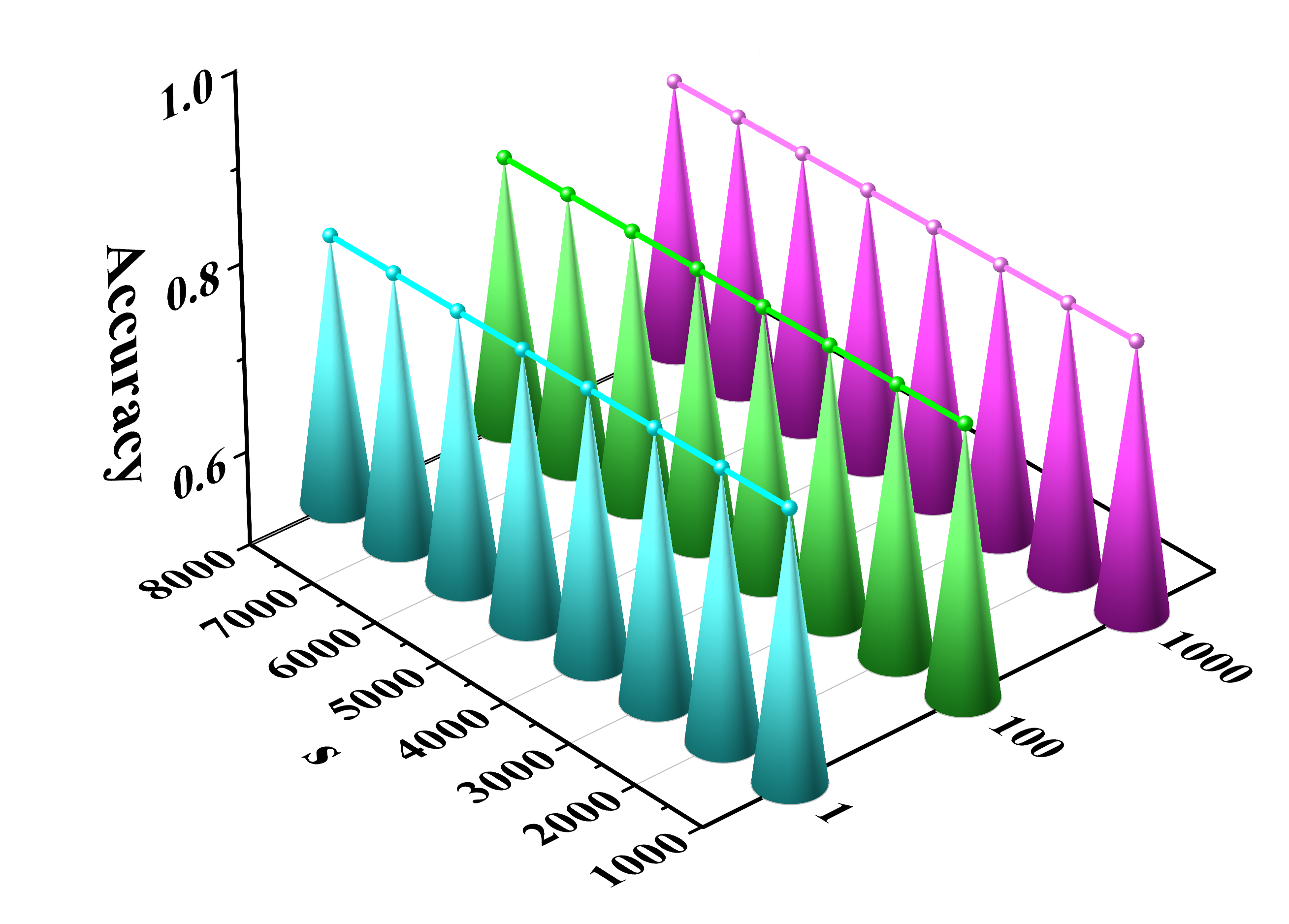}
    
    \includegraphics[width=0.32\linewidth]{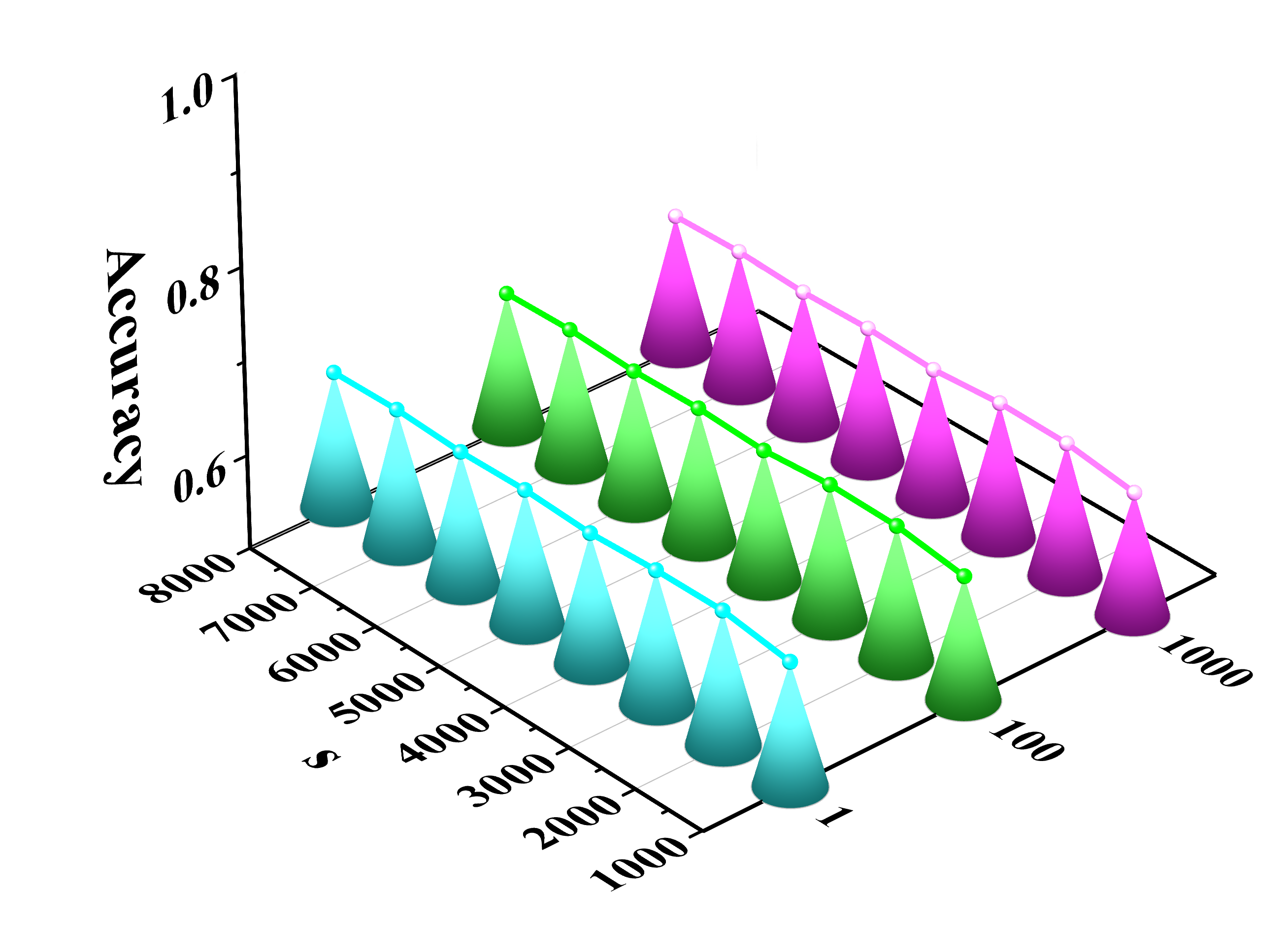}
    \includegraphics[width=0.32\linewidth]{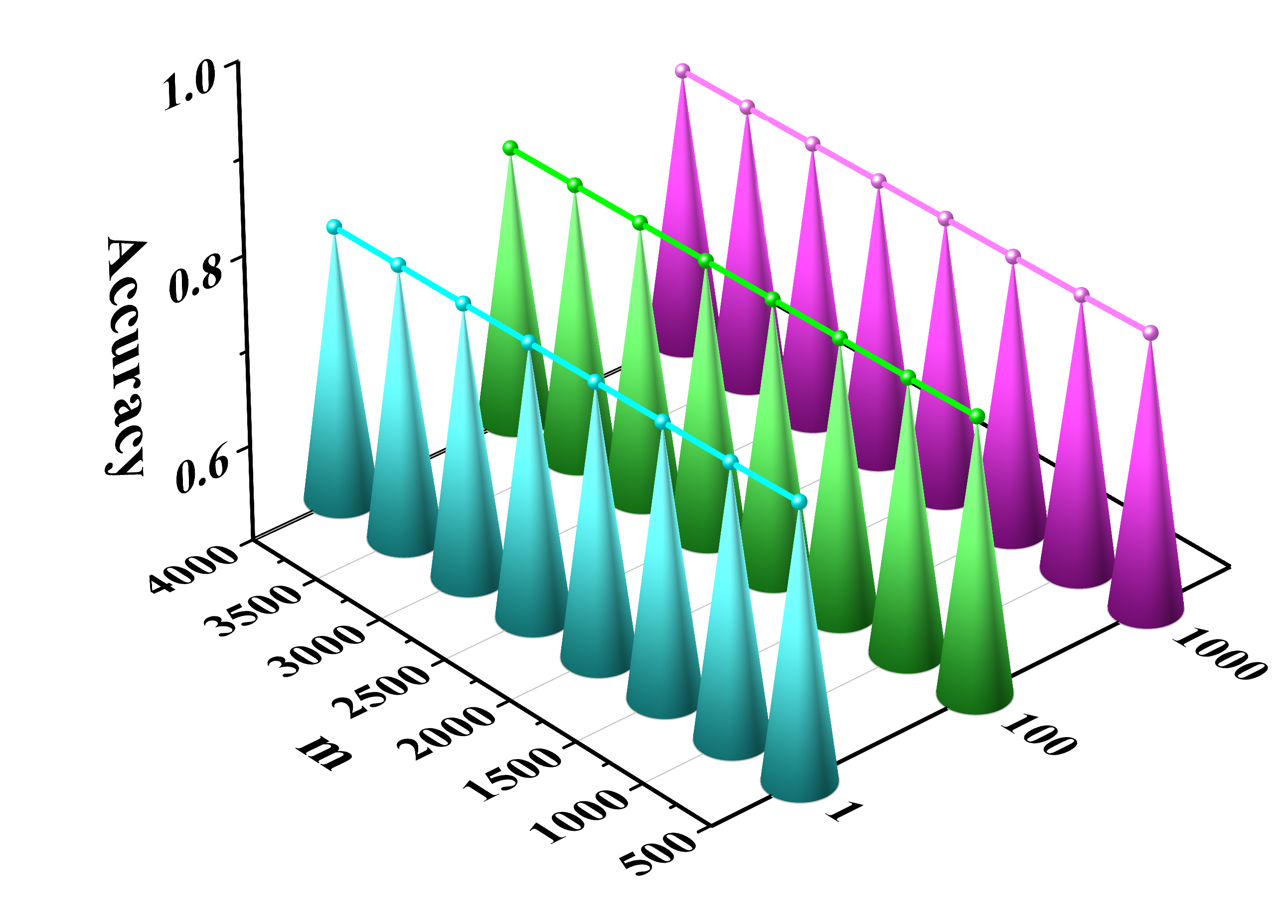}
    \includegraphics[width=0.32\linewidth]{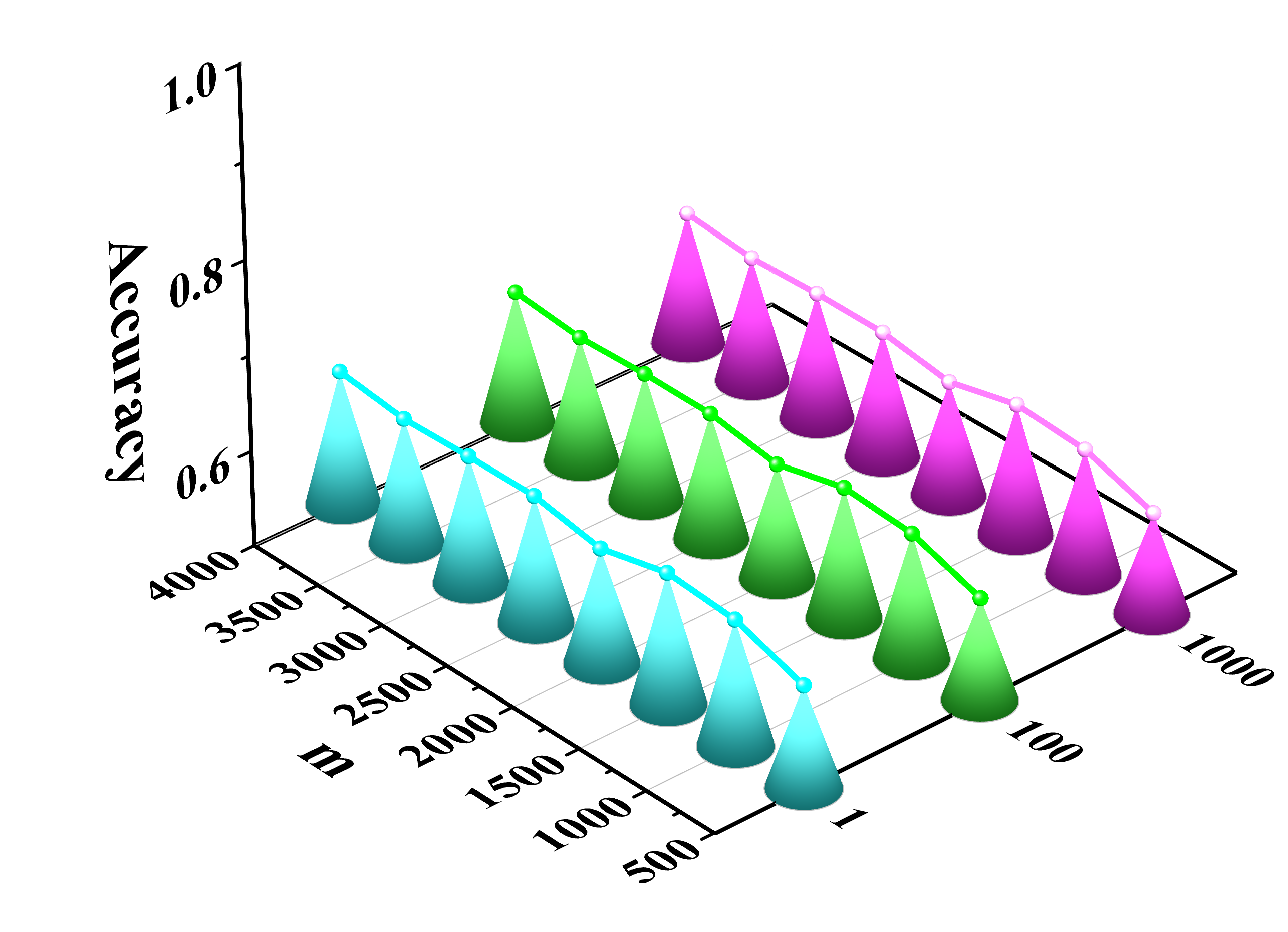}
    \caption{Parameter sensitivity analysis of $\psi$ (top), $s$ (middle), $m$ (bottom) on SUSY (left) and HIGGS (right) for LR.}
    \label{fig: para_lr}
\end{figure*}

\begin{figure*}[!ht]
    \centering
    \includegraphics[width=0.32\linewidth]{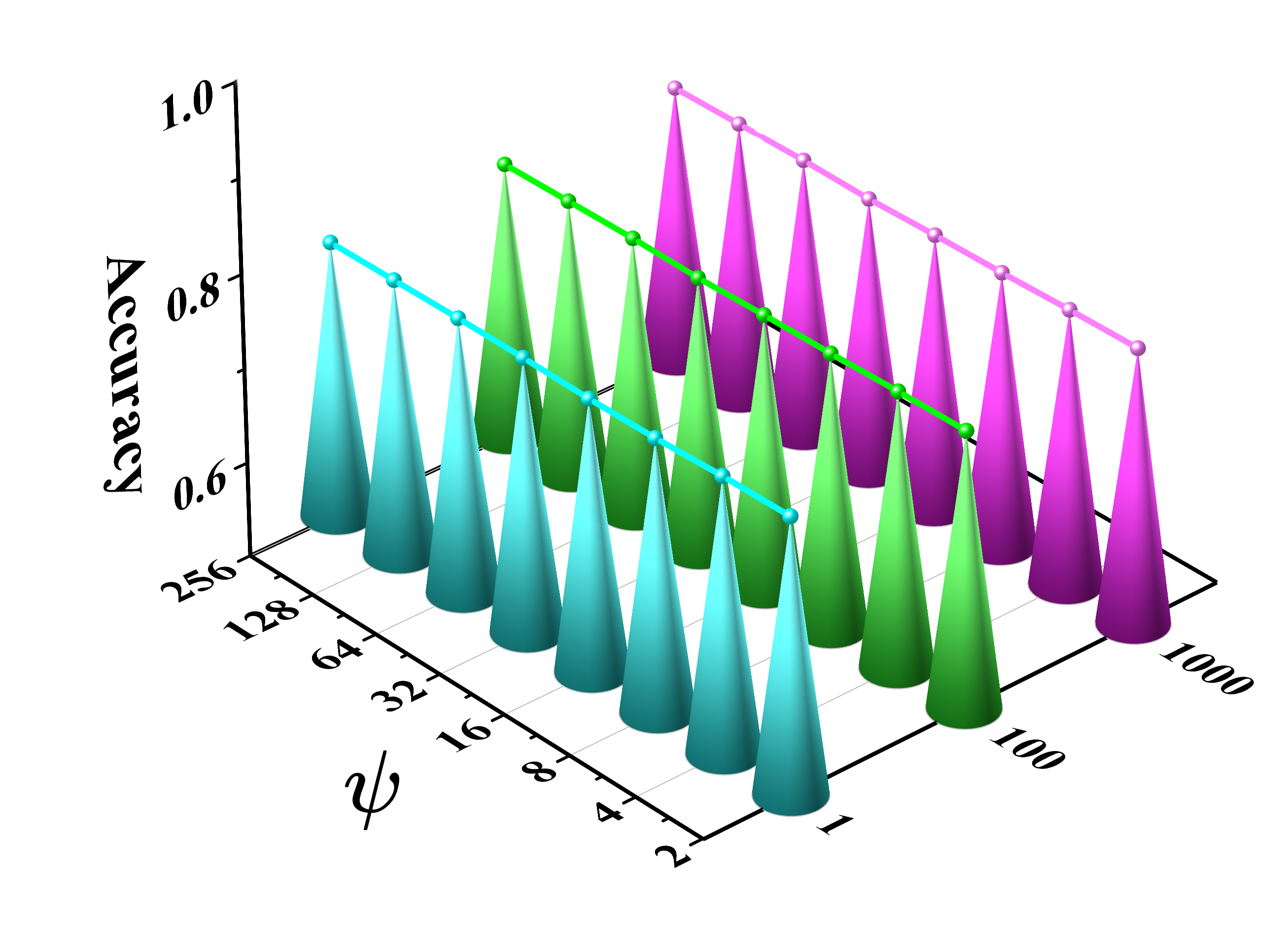}
    \includegraphics[width=0.32\linewidth]{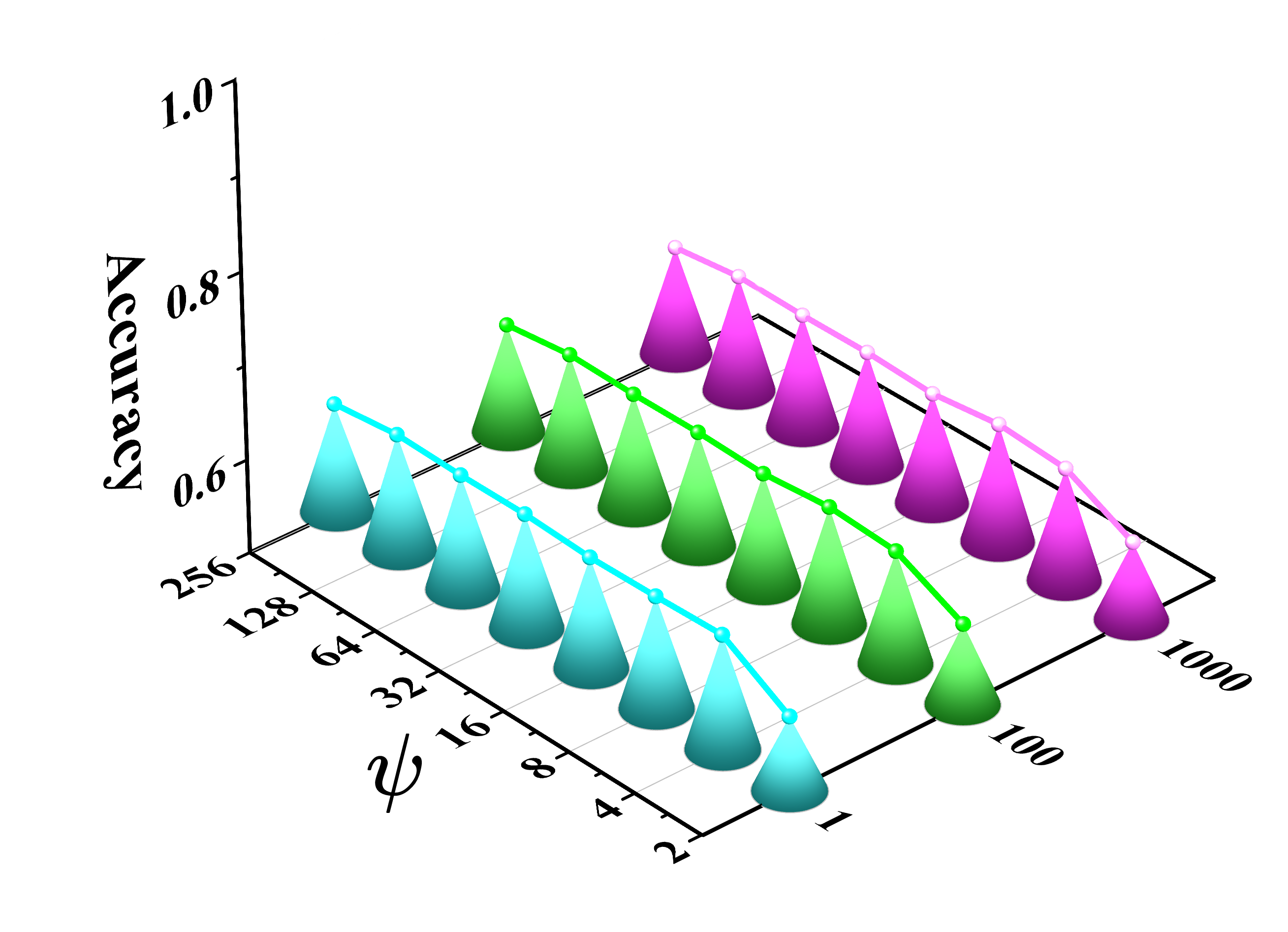}
    \includegraphics[width=0.32\linewidth]{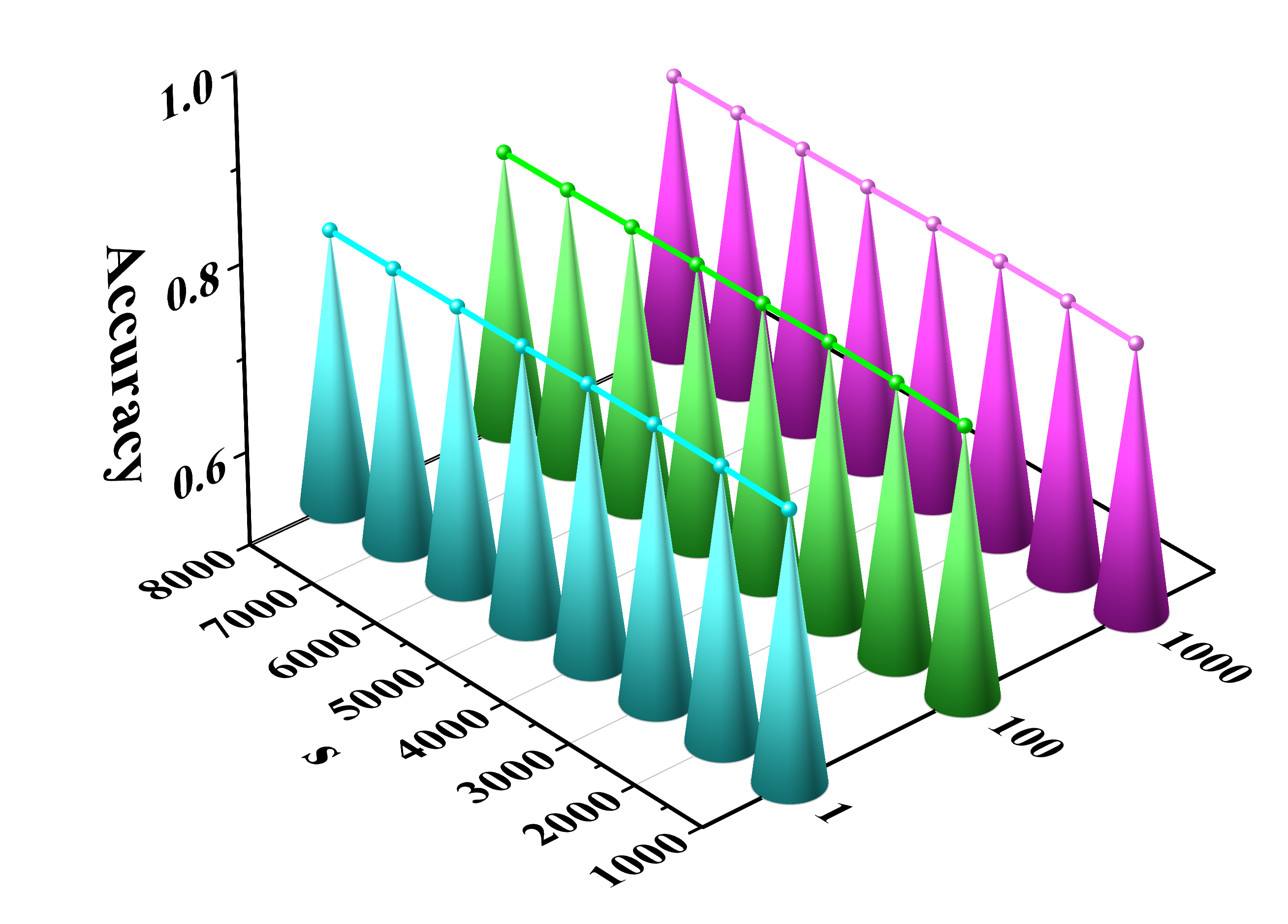}
    
    \includegraphics[width=0.32\linewidth]{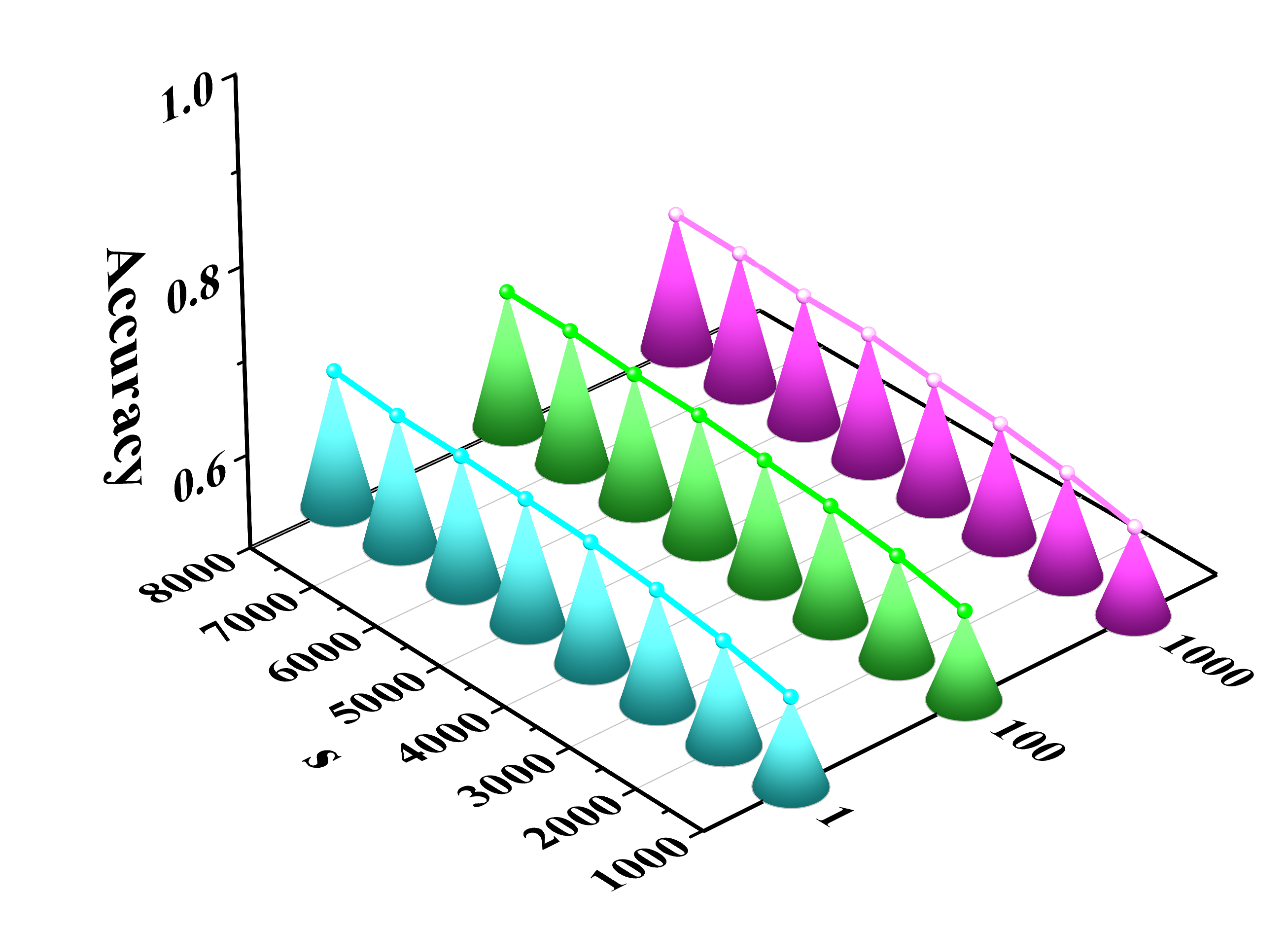}
    \includegraphics[width=0.32\linewidth]{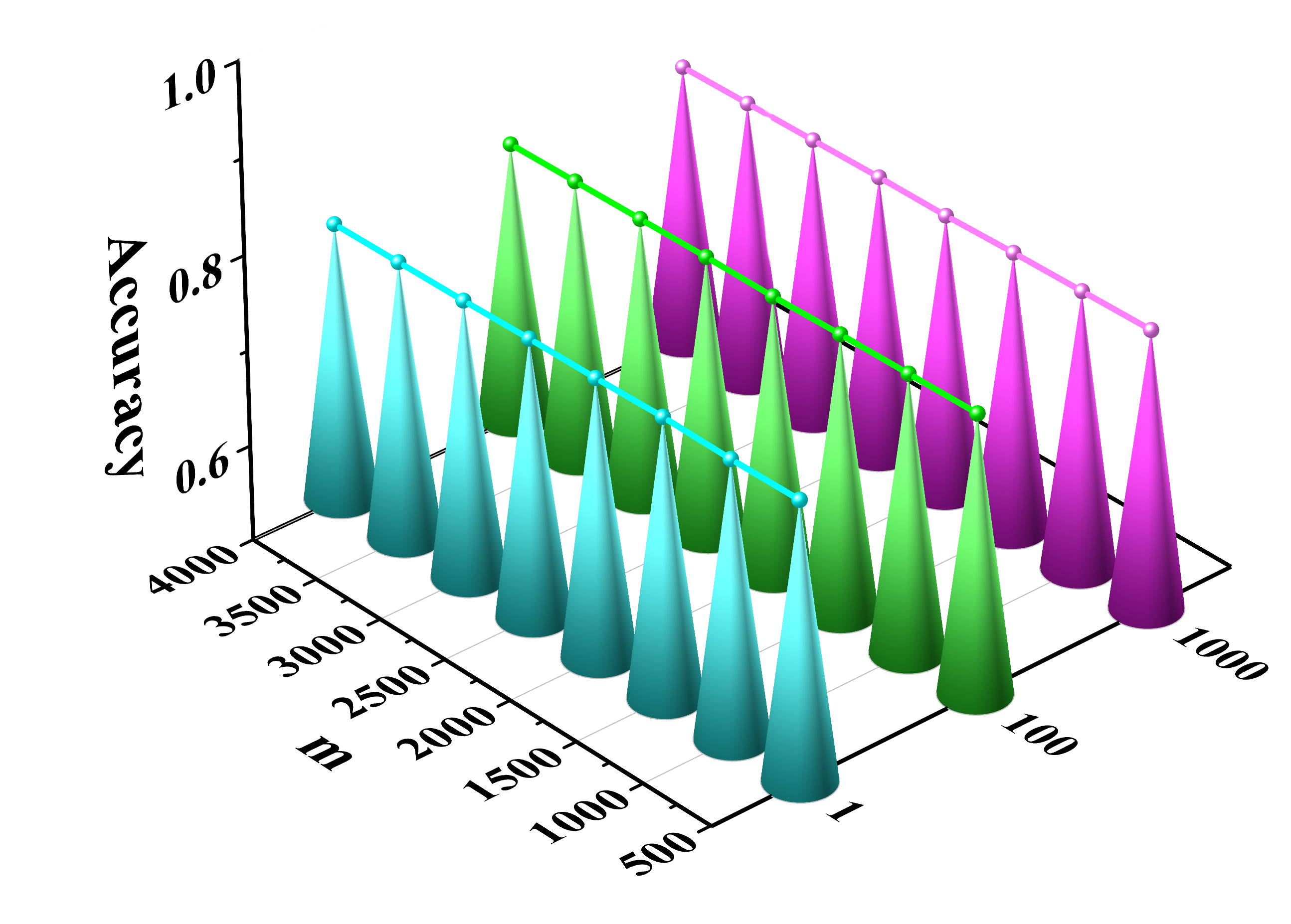}
    \includegraphics[width=0.32\linewidth]{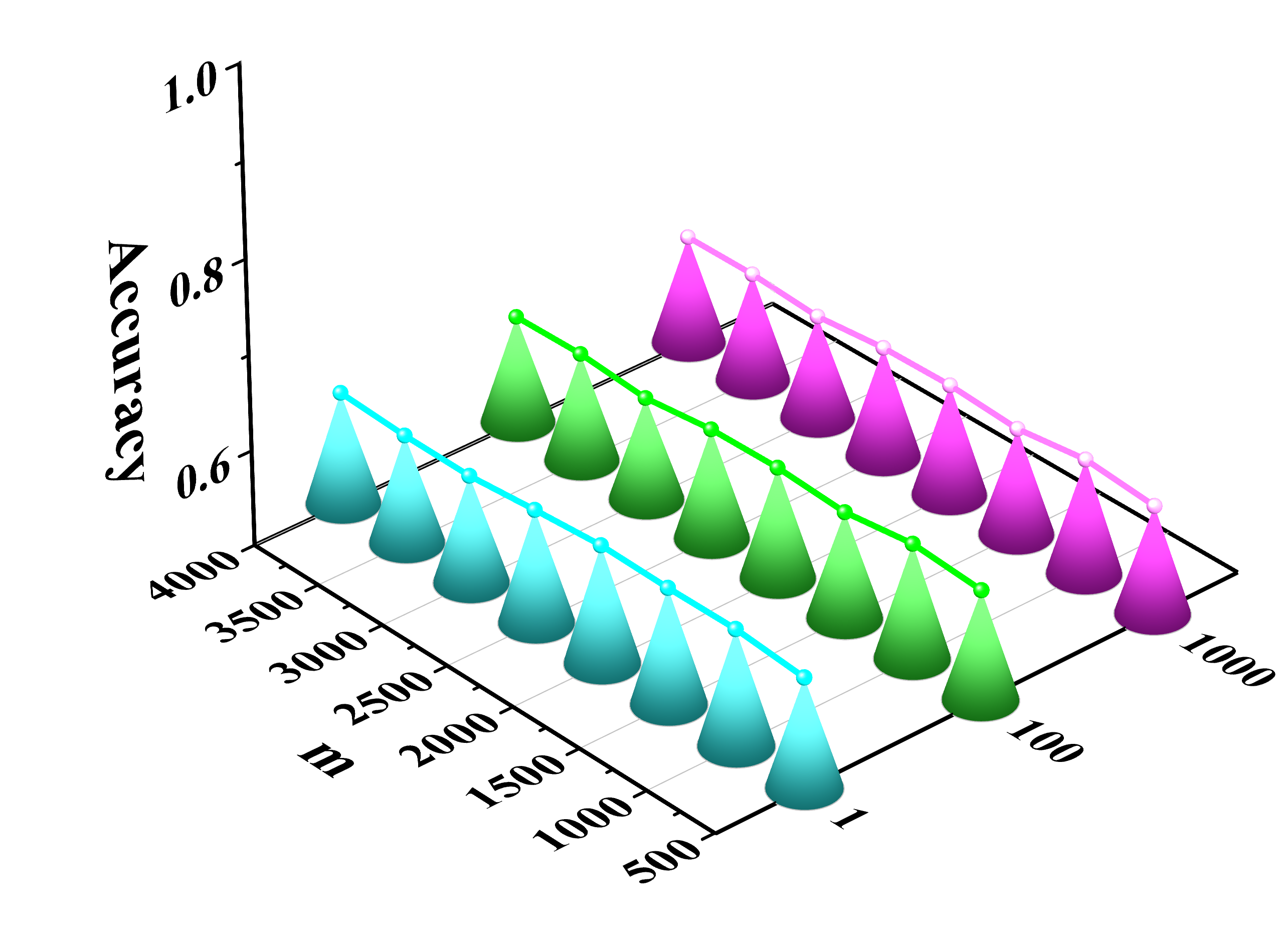}\label{fig: para_m}
    \caption{Parameter sensitivity analysis of $\psi$(top), $s$(middle), $m$(bottom) on SUSY (left) and HIGGS (right) for MLP.}
    \label{fig: para_mlp}
\end{figure*}

The results of $\psi$ are shown in the first row of Figure \ref{fig: para_lr} and Figure \ref{fig: para_mlp}, L2UL is robust to $\psi$ when $\psi\geq 4$.
The results of $s$ are shown in the second row of Figure \ref{fig: para_lr} and Figure \ref{fig: para_mlp}. 
For LR on SUSY, HIGGS, and MLP on SUSY, L2UL is robust to $s$, and the accuracy of L2UL is close to that of retrain. 
For MLP on HIGGS, as $s$ increases, the accuracy of L2UL increases from 0.60 to 0.65. This is because when $s$ (sample size) is small, we cannot get a good model $A(\mathfrak{D})$ and $A(\mathfrak{D}\setminus\mathfrak{D}_f)$ as input into L2UL. When sample size $s$ increases, we can have a better model $A(\mathfrak{D})$ and $A(\mathfrak{D}\setminus\mathfrak{D}_f)$, so that a good $U$ and unlearned model $U(D, D_f, A(D))$ with higher accuracy can be obtained on HIGGS.

The results of $m$ are shown in the third row of Figure \ref{fig: para_lr} and Figure \ref{fig: para_mlp}. L2UL is robust to $m$, as $N_U$ is a simple model that can achieve good performance with a small dataset $\mathscr{D}$.

In a nutshell, L2UL is robust to the parameters $\psi$ and $m$. Compared to the parameters of the L2UL model and the parameters of IDK, L2UL requires stronger data as input. 
L2UL learns an unlearning function $U$, only if the input data $\mathscr{D}$ contains sufficient unlearning information \footnote{Sufficient unlearning information means a correct and well-trained model $A(\mathfrak{D})$ and a correct and well-unlearned model $A(\mathfrak{D}\setminus\mathfrak{D}_f)$.}, can a correct and good unlearning function be learned.

\section{Complexity Analysis}

\begin{table}[t!]
\centering
\caption{Factor of time complexity ($D$-or: $D$-oriented, $D_f$-or: $D_f$-oriented, $A(D)$-or: $A(D)$-oriented).}
\label{tab: time-com}
\setlength{\tabcolsep}{4mm}{
\begin{tabular}{l|llll}
\toprule 
$U$ & $D$-or & $D_f$-or  & $A(D)$-or   & L2UL \\ \midrule
Factor                & $A\& D$                 & $D_f$                      & $D_f \&\ \theta$                & $D_f$   \\
\bottomrule 
\end{tabular}}
\end{table}

\begin{table}[h]
\centering
\caption{Results of unlearning ResNet-18 on CIFAR-10 dataset.}
\label{tab:resnetcifar10}
\setlength{\tabcolsep}{1mm}{
\begin{tabular}{ccrrrr}
\toprule
		                          &                     & 10         & 100    & 1000       & 3000        \\ \midrule
\multirow{4}{*}{ToW($\uparrow$)}    & Retrain   & 1.00       & 1.00   & 1.00       & 1.00           \\
		                          & ADV+IMP   & 0.32       & 0.13   & 0.01       & 0.01        \\
		                          & RUM     & 0.60       & 0.80    & 0.83       & 0.83        \\
\rowcolor{gray!20}		            & L2UL(Ours)          & 0.80       & 1.00   & 0.88       & 0.99       \\ \midrule
\multirow{4}{*}{MIA gap($\downarrow$)}  & Retrain   & 0.00       & 0.00   & 0.00       & 0.00           \\
		                          & ADV+IMP    & 0.30       & 0.13   & 0.17       & 0.14        \\
		                          & RUM     & 0.30       & 0.13   & 0.15       & 0.12        \\
\rowcolor{gray!20}		            & L2UL(Ours)          & 0.30       & 0.13   & 0.17       & 0.14        \\  \midrule
\multirow{5}{*}{Time($\downarrow$)} & Origin ($A(D$)      & 1044.9    &   -    & -          & -            \\
		                          & Retrain    & 1020.5    & 874.4 & 783.2     & 702.7       \\
		                          & ADV+IMP    & 28.4      & 78.7  & 1665.5    & 12088.2    \\
		                          & RUM     & 13.6      & 13.6   & 13.2      & 12.8        \\
\rowcolor{gray!20}		            & L2UL(Ours)          & 0.5       & 0.5   & 0.5       & 0.6 \\ \bottomrule
\end{tabular}}
\end{table}

We show the factor that affects the time complexity of unlearning algorithms as follows (time complexity of the machine learning model $A$ is denoted as $\mathcal{T}(A)$) :
\begin{enumerate}
    \item $D$-oriented unlearning algorithm, such as SISA \cite{SISA}, has a time complexity of \text{$\mathcal{T}(A(|D|/k))$}, where $k$ is the number of shards. Because SISA retrains $A$ on some shards, the time complexity is limited by $A$ and $|D|$. 
    \item $D_f$-oriented unlearning algorithm, such as \cite{guo2020certified}, considers the influence of $D_f$, whose time complexity is $\mathcal{O}(|D_f|d^3)$, which is limited by $D_f$. 
    \item $A(D)$-oriented unlearning algorithm, such as \cite{sekhari2021remember}, considers how to update parameters. Its time complexity is $\mathcal{O}(|D_f|d^2+d^{|\theta|})$ and is limited by $D_f$ and $|\theta|$.
\end{enumerate}

The factors of time complexity are summarized in Table \ref{tab: time-com}.
Intuitively, $D_f$ is unavoidable because we at least need to know the samples we need to forget.  L2UL is also subject to this restriction. Given dataset $D$, we give the time complexity and space complexity of L2UL as follows:

\begin{enumerate}
    \item[*]Time complexity of preprocessing is $\mathcal{O}(\psi t |D| d+m\mathcal{T}(A(s)))$
    \item[*]Time complexity of training $U$ is $\mathcal{T}(U(m))$.
    \item[*] Time complexity of unlearning $D_f$ from $U$ is $\mathcal{O}(\psi t |D_f| d)$.
    \item[*] Space complexity of $U$ is $\mathcal{O}(\psi t d)$.
\end{enumerate}

When unlearning $D_f$, L2UL needs to calculate the feature map of $D_f$ first. But we can store the feature map of every point from $D$ in advance. In this way, when unlearning $D_f$, we only need to select the corresponding feature maps and take the mean. Although the time complexity still depends on $D_f$, the operation of just taking mean is very fast. As a trade-off, space complexity goes up to $\mathcal{O}(\psi t |D|)$.

\section{Unlearn large-scale parameters: ResNet}

This paper and the method L2UL we propose primarily focus on the unlearning of simple models (with fewer parameters) trained on large amounts of training data. As demonstrated in the previous sections, the experimental results confirm that L2UL not only maintains high fidelity to the retrained model but also achieves significant gains in computational efficiency. Furthermore, in this section, we verify the effectiveness of our proposed L2UL on a model with a larger parameter size. We tested the performance of the unlearning ResNet-18 model on the CIFAR-10 dataset and compared it with:

\textbf{ADV+IMP: }\cite{cha2024learning} Misclassify each instance outside of its original prediction, or relabel the instance to a different label. At the same time, use adversarial examples to overcome forgetting at the representation level and use weight importance indicators to accurately locate network parameters that propagate unnecessary information to reduce the time required for forgetting.

\textbf{RUM: }\cite{zhao2024makes} The forgotten set is refined into homogenized subsets based on different features. The meta-algorithm employs the existing algorithms to forget each subset, ultimately providing a model that has forgotten the entire forgotten set.

We train ResNet-18 for 50 epochs on CIFAR-10 with a learning rate of 0.0001. And we report two evaluation metrics, ToW and MIA (Membership Inference Attack) gap, that are used in RUM \cite{zhao2024makes}.

The results are shown in Table \ref{tab:resnetcifar10}. We highlight the following \underline{three key observations}: (a) L2UL achieves the highest ToW scores when forgetting 10, 100, 1000, and 3000 samples. (b) L2UL and the comparison algorithms have a closed low MIA gap. (c) L2UL only takes a very short time to complete the unlearning.

In summary, our experimental results demonstrate not only the scalability of L2UL on large-scale data but also its scalability in terms of model parameters.  L2UL maintains the performance of the model after forgetting data, and, more importantly, its unlearning efficiency is far superior to existing algorithms.

\section{Disscussion}
\label{sec: disscussion}

\subsection{How to learn $U$ when the model cannot be retrained?}

It works well for L2UL to use Equation \ref{eqa: mse} as its loss function only when the parameter of the retrained model $A(\mathfrak{D}\setminus \mathfrak{D}_f)$ can be obtained. 
However, the inability to retrain the model is a common challenge faced by machine unlearning. For example, we can't retrain the model on streaming data.
In this case, L2UL can no longer be used to generate training data $\mathscr{D}$, and the loss function \ref{eqa: mse} cannot be used to train $N_U$. How to find a new loss function when the model cannot be retrained is a challenge. 

\subsection{Retrain $U$ after unlearning $A$}

It is worth emphasizing that when we employ $U$ to unlearn model, user's information $D_f$ will be included in $U$. Therefore, we must unlearn $U$ after we use $U$ to unlearn $A$.

\begin{theorem}
\label{theorem: prob_sample}
    When sampling $m$ subsets $\mathfrak{D}$ from $D$. The probability that $x$ is sampled at least $k_0$ times is \[
    \mathcal{P}(k\geq k_0)=\sum_{i=k_0}^m C_m^k(\frac{s}{|D|})^i(1-\frac{s}{|D|})^{m-i}.\]
\end{theorem}

In our experiments, $\mathcal{P}(k\geq 2)<0.004$ on HIGGS, which means the probability that the data $\{x, y\}\in D$ is used to train $U$ is very low. So we barely need to retrain $U$. Even if we need, the time cost is affordable, as it only takes 5.2 seconds to retrain $U$ on HIGGS.

\subsection{Subdivide the machine unlearning task}
Current machine unlearning methods often treat forgetting tasks as a monolithic problem, neglecting the inherent diversity in model complexity and data scale. We argue for a more nuanced, task-specific approach. Since a forgetting task is defined by both the model and the training data, we propose a taxonomy that categorizes machine unlearning into four distinct scenarios based on model size and dataset size:

1) Simple models on small datasets. This represents the most straightforward scenario, where retraining from scratch remains computationally feasible and is often the preferred solution.

2) Complex models on small datasets. Although the training data is limited, the increased model complexity significantly raises the cost of retraining, making it an impractical choice.

3) Simple models on large datasets. This is the focal scenario of our work. Because the model itself is lightweight, applying sophisticated unlearning techniques designed for complex architectures would be unnecessarily resource-intensive—essentially overkill. Instead, a lightweight method, such as the one we propose, is sufficient to achieve effective forgetting.

4) Complex models on large datasets (e.g., large language models). This constitutes the most challenging scenario, demanding advanced and scalable unlearning strategies.

The key advantage of this taxonomy lies in its strategic flexibility: by matching the unlearning method to the specific complexity of the task, we can achieve efficient data removal without incurring prohibitive computational costs, ensuring we neither "use a sledgehammer to crack a nut" nor deploy inadequate solutions for demanding scenarios.

\section{Conclusion}

In this paper, we propose Learning-to-UnLearn (L2UL), to the best of our knowledge, the first framework to address machine unlearning by learning the unlearning behaviors rather than manually designing complex functions. By prioritizing a distribution-oriented perspective, L2UL achieves a remarkably streamlined yet potent unlearning capability. Focusing on scenarios where the model is simple but the data training dataset is large, we utilized Logistic Regression and Multilayer Perceptrons to demonstrate the ability of L2UL to handle massive datasets with exceptional speed and no loss in accuracy. Both theoretical and empirical results prove that L2UL significantly outperforms existing methods in efficiency. Furthermore, tests on contaminated models and sensitivity analyses confirm the effectiveness and stability of L2UL across various configurations. Finally, successful validation on ResNet underscores the scalability of L2UL to more complex models, offering a robust and efficient solution for modern privacy requirements.




\appendix
\section{Proofs}
\label{appendix: proof}
\subsection{Proof of Theorem 4.5}

\textbf{Theorem 4.5}
    \begin{equation*}
        L(f_{U(A(D), D, D_f)})\leq L(f_{A(D\setminus D_f)}) + 2R\sqrt{\epsilon},
    \end{equation*} 
    where $R$ is the radius of the dataset $X$ ($\forall x\in X, ||x||\leq R$), and $\epsilon$ is the generalization bound (MSE) of the unlearning model $U$.

\begin{proof}
    Denote $U(A(D), D, D_f)-A(D\setminus D_f)$ as $e$, then $||e||\leq\sqrt{\epsilon}$. By Taylor Expansion, we approximately have
    \begin{equation*}
    \begin{aligned}
        & L(f_{U(A(D), D, D_f)})-L(f_{A(D\setminus D_f)}) \\
        = &\frac{\partial L(f_{A(D\setminus D_f)})}{ \partial A(D\setminus D_f)}^\top e+o(\|e\|),
    \end{aligned}
    \end{equation*}
    where $\lim_{\|e\|\rightarrow 0}\frac{o(\|e\|)}{\|e\|}=0$. Then by Cauchy–Schwarz inequality and convexity of norm function, we have
    \begin{equation*}
        \begin{aligned}
             & L(f_{U(A(D), D, D_f)})-L(f_{A(D\setminus D_f)})\\
             = &E_{x,y}[(f_{A(D\setminus D_f)}(x)-y) x^\top] e+o(\|e\|)\\
             \leq & ||E_{x,y}[(f_{A(D\setminus D_f)}(x)-y) x]||\cdot||e||+o(\|e\|)\\
            \leq & E_{x,y}[||(f_{A(D\setminus D_f)}(x)-y)x||]\cdot||e||+o(\|e\|) \\
            \leq & 2R\sqrt{\epsilon} +o(\sqrt{\epsilon}).
        \end{aligned}
    \end{equation*}
\end{proof}

Here we refer $U(A(D), D, D_f)$ and $A(D\setminus D_f)$ as the parameters of the unlearned model and retrained model, respectively.

\subsection{Proof of Theorem 4.6}

\textbf{Theorem 4.6}
    \begin{equation*}
        L(f_{U(A(D), D, D_f)})\leq L(f_{A(D\setminus D_f)}) + A\sqrt{\epsilon},
    \end{equation*} 
    where $A$ is a finite scalar determined by the MLP structure and $\epsilon$ is the generalization bound (MSE) of the unlearning model $U$.

\begin{proof}
For the ($N+1$)-layer MLP model (the 0-th layer is the input layer, the ($N+1$)-th layer is the output layer), we use softmax activation for the output layer, and sigmoid activation function for the rest of the layers. We denote the number of neurons in the i-th layer as $n_i$, then the weight between the $(i-1)$-th layer and $i$-th layer are $W_i$ of shape $(n_k,n_{k-1})$ and $B_i$ of shape $(n_i,1)$. For a finite test set $X$, $|X|=m$, we denote it as $A_0$. Then, the forward propagation scheme is described as follows:
\[
    Z_k = W_k A_k + B_k,\ A_k = f_k(Z_k)
\]
for $k\in\{1,2,..., N+1\}$, where $f_k(\cdot)$ is the activation function we use in the k-th layer (softmax for the output layer, sigmoid for the rest layers).

For the back propagation, when using Cross-Entropy Loss, we have
\[
    d\;{Z_{N+1}} = A_{N+1}-Y, \    d\;W_{N+1} = \frac{1}{m}   d\;Z_{N+1} A_N^\top,
\]
\[
    d\; B_{N+1} = \frac{1}{m} d\;Z_{N+1} 1_{(m,1)},
\]
and 
\begin{equation}
    d\;Z_k = W_{k+1}^T d\;Z_{k+1} \odot f^\prime_k(Z_k),
    \label{equ: backprop Z}
\end{equation}
\[
    d\; W_k = \frac{1}{m}d\;Z_k A^\top_{k-1}, \
    d\;B_k = \frac{1}{m}d\;Z_k 1_{(m,1)}
\]
for $k\in\{1,2,...,N,N+1\}$,
where $d$ means the derivative of the Cross Entropy Loss, $1_{(m,1)}$ is a column vector of size m, with each entry $=1$, and $\odot$ is the element-wise multiplication. 

Denote  $A(D\setminus D_f)$ as $\Theta$ and $U(A(D), D, D_f)$ as $\hat{\Theta}$. We have
\begin{equation}
    \Theta = \begin{bmatrix}
  W\\B
\end{bmatrix},
\end{equation}
where $W, B$ are the concatenation of each layer's row-flatten parameter, i.e. $W=\left [ \text{rft}(W_1),...,\text{rft}(W_{N+1})\right ]^\top $,$W=\left [ \text{rft}(B_1),...,\text{rft}(B_{N+1})\right ]^\top $, where $\text{rtf}(\cdot)$ is the operator that flatten a matrix into a row vector. The same notation goes for $\hat{\Theta}$. As in     \ref{theorem: logistic Unlearn error}, we only need to show that $d\;\Theta$ is bounded, which can be achieved from the bound for $d\; B$ and $d\ ;W$, since $\|d\;\Theta\|^2=\|d\; W\|^2+\|d\; B\|^2$.

For $W_{N+1}$, we have 
\[
    \|d\;\text{rft}({W_{N+1}})\|^2 = \|\frac{(A_{N+1}-Y)A_N^\top}{m}\|_F^2 \leq 4n_N n_{N+1}.
\]
since the absolute value of each entry in $A_{N+1},A_N,Y$ is smaller than 1. For the bound on $d\;\text{rft} W_k(k\neq N+1)$, we first need to get the bound $R_k$ for $d\;\text{rft}Z_k$. We know that 
\[
    \|d\;\text{rft}(Z_{N+1})\|_\infty\leq R_{N+1}=2.
\]
Then according to Equation \ref{equ: backprop Z}, we have 
\[
    \|d\;\text{rft}Z_k\|_\infty\leq \frac{R}{4}n_{k+1} R_{k+1}\triangleq R_k,
\]
under the assumption $\|W\|_\infty\leq R$. The iterative form of $R_k$ can be expressed equivalently as 
\[
    R_k = (\frac{R}{4})^{N+1-k}\prod_{j=k+1}^{N+1}n_j .
\]
Coming back to $W_k$, for $k\leq N$, we have 
\[
\|d\;\text{rft}(W_k)\|^2\leq R_k^2 n_k n_{k-1}.
\]
Finally, we have
\begin{equation*}
    \begin{aligned}
        \|d\;\text{rft}W\|^2 &= \sum_{k=1}^{N+1} \|d\;\text{rft}W_k\|^2\\
        &\leq \underbrace{4n_N n_{N+1}+\sum_{k=1}^N (\frac{R^2}{16})^{N+1-k}\frac{\prod_{j=k-1}^{N+1}n_j^2}{n_k n_{k+1}}}_{\triangleq A_W}.
    \end{aligned}
\end{equation*}
Similarly, we have
\begin{equation}
    \begin{aligned}
        \|d\;\text{rft}(W_k)\|^2\leq R_k^2 n_k n_{k-1}.
    \end{aligned}
\end{equation}
Finally, we have
\begin{equation}
    \begin{aligned}
        \|d\;\text{rft}B\|^2 &= \sum_{k=1}^{N+1} \|d\;\text{rft}B_k\|^2\\
        &\leq \underbrace{4n_{N+1}+\sum_{k=1}^N (\frac{R^2}{16})^{N+1-k}\frac{\prod_{j=k}^{N+1}n_j^2}{n_k}}_{\triangleq A_B}.
    \end{aligned}
\end{equation}
Then $\|d\;\Theta\|\leq\sqrt{A_w^2+A_B^2}\triangleq A$.
\end{proof}

\subsection{Proof of Complexity}
\label{appendix: complexity}
\begin{enumerate}
    \item[*]Time complexity of preprocessing is $\mathcal{O}(\psi t |D| d+m\mathcal{T}(A(s)))$

    \begin{proof}
        Preprocessing first requires a total time of $\mathcal{O}(\psi t |D| d$ to produce the feature map (line 2 in Algorithm \ref{alg: preprocessing}), and requires $\mathcal{O}(\psi ts+\mathcal{T}(A(s))$ for each loop, So the time complexity of Algorithm \ref{alg: preprocessing} is $\mathcal{O}(\psi t |D| d+m\mathcal{T}(A(s)))$.
    \end{proof}
    
    \item[*]Time complexity of training $U$ is $\mathcal{T}(U(m))$.

    \begin{proof}
     The time complexity of Train U is using the L2UL to train $N_U$ on the $m$ data points, so the time complexity is  $\mathcal{T}(U(m))$. 
\end{proof}
    \item[*] Time complexity of unlearning $D_f$ from $U$ is $\mathcal{O}(\psi t |D_f| d)$.

\begin{proof}
     Algorithm \ref{alg: unlearn} requires time of $\mathcal{O}(\psi t |D_f| d)$ to get the feature map (line 1). and $\mathcal{O}(\psi t |D_f|)$ to get $\mu_f$ (line 2). Therefore, the time complexity to unlearn $D_f$ from the $U$ is $\mathcal{O}(\psi t |D_f| d)$.
     \end{proof}
    \item[*] Space complexity of $U$ is $\mathcal{O}(\psi t d)$.

    \begin{proof}
    we need $\mathcal{O}(\psi t)$ to store the $\mu_D$, and $\mathcal{O}(\psi td)$ to store the seeds for partition in isolation kernel. Therefore,  the space complexity of $U$ is $\mathcal{O}(\psi t d)$.
\end{proof}

\end{enumerate}

\subsection{Proof of Theorem 7.1}
\textbf{Theorem 7.1}

We sample $m$ subsets $\mathfrak{D}$ from $D$. The probability that sample $x$ is sampled at least $k_0$ times is: \[
    \mathcal{P}(k\geq k_0)=\sum_{i=k_0}^m C_m^k(\frac{s}{|D|})^i(1-\frac{s}{|D|})^{m-i}.
    \]

\begin{proof}
    Sampling $s$ points from $|D|$ points, The probability that point $x$ is sampled is $\mathcal{P}(x)=\frac{s}{|D|}$.

    Do this sampling $m$ times, the probability that point $x$ is sampled $k$ times is:
    \[
    P(k)=C_m^k(\frac{s}{|D|})^k(1-\frac{s}{|D|})^{m-k}.
    \]

    So the probability that sample $x$ is sampled at least $k_0$ times is: \[
    \mathcal{P}(k\geq k_0)=\sum_{i=k_0}^m C_m^k(\frac{s}{|D|})^i(1-\frac{s}{|D|})^{m-i}.
    \]
    
\end{proof}



\bibliography{ref}
\bibliographystyle{ACM-Reference-Format}

\end{document}